%% file: main.tex
\documentclass{article}

\input{kern-tex/headers.tex}
\input{kern-tex/defs}

\usepackage[accepted]{icml2021}

\icmltitlerunning{Scalable Variational Gaussian Processes via Harmonic Kernel Decomposition}

\begin{document}

\twocolumn[
\icmltitle{Scalable Variational Gaussian Processes via Harmonic Kernel Decomposition}




\begin{icmlauthorlist}
\icmlauthor{Shengyang Sun}{to,vec}
\icmlauthor{Jiaxin Shi}{mr}
\icmlauthor{Andrew Gordon Wilson}{nyc}
\icmlauthor{Roger Grosse}{to,vec}
\end{icmlauthorlist}

\icmlaffiliation{to}{University of Toronto}
\icmlaffiliation{vec}{Vector Institute}
\icmlaffiliation{mr}{Microsoft Research New England}
\icmlaffiliation{nyc}{New York University}

\icmlcorrespondingauthor{Shengyang Sun}{ssy@cs.toronto.edu}
%\icmlcorrespondingauthor{Eee Pppp}{ep@eden.co.uk}

\icmlkeywords{Machine Learning, ICML}

\vskip 0.3in
]



\printAffiliationsAndNotice{} % otherwise use the standard text.

\begin{abstract}

We introduce a new scalable variational Gaussian process approximation which provides a high fidelity approximation while retaining general applicability. We propose the harmonic kernel decomposition (HKD), which uses Fourier series to decompose a kernel as a sum of orthogonal kernels. Our variational approximation exploits this orthogonality to enable a large number of inducing points at a low computational cost. We demonstrate that, on a range of regression and classification problems, our approach can exploit input space symmetries such as translations and reflections, and it significantly outperforms standard variational methods in scalability and accuracy. Notably, our approach achieves state-of-the-art results on CIFAR-10 among pure GP models.

\end{abstract}

\input{kern-tex/intro.tex}

\input{kern-tex/background}

\input{kern-tex/kernel.tex}
\input{kern-tex/spectral-decoupled-inference}

\input{kern-tex/related-works}
\input{kern-tex/experiments}
\section{Conclusion}
We presented the harmonic kernel decomposition which exploited input-space symmetries to obtain an orthogonal kernel sum decomposition, based on which we introduced a scalable variational GP model and analyzed how well the model approximates the true posterior of the GP. We validated its superior performances in terms of scalability and accuracy through a range of empirical evaluations.

%We introduce a scalable variational approximation for Gaussian processes via the \emph{harmonic kernel decomposition}, which decomposes the kernel as a summation of orthogonal kernels using Fourier series. 

\section*{Acknowledgements}
We thank Wesley Maddox, Greg Benton, Sanyam Kapoor, Michalis Titsias, Radford Neal, and anonymous reviewers for their insightful comments and discussions on this project. We also thank the Vector Institute for providing the scientific computing resources. SS was supported by the Connaught Fellowship. RG acknowledges support from the CIFAR Canadian AI Chairs program.

\bibliography{bib/gp,bib/bnn,bib/gan,bib/vi,bib/misc,bib/bayesian,bib/nn,bib/generalization,bib/rl,bib/logic,bib/cv,bib/nlp,bib/matrix}
\bibliographystyle{icml2021}


\clearpage
\onecolumn
\appendix
\input{kern-tex/appendix}

%%%%%%%%%%%%%%%%%%%%%%%%%%%%%%%%%%%%%%%%%%%%%%%%%%%%%%%%%%%%%%%%%%%%%%%%%%%%%%%
%%%%%%%%%%%%%%%%%%%%%%%%%%%%%%%%%%%%%%%%%%%%%%%%%%%%%%%%%%%%%%%%%%%%%%%%%%%%%%%

\end{document}

%% file: kern-tex/headers.tex
\usepackage[utf8]{inputenc} 
\usepackage{amsmath}
\usepackage{amsthm}
\usepackage{amssymb,bm, cases, mathtools, thmtools}
\usepackage{verbatim}
\usepackage{graphicx}\graphicspath{{figures/}}
\usepackage{multicol}
\usepackage{tabularx}
\usepackage[usenames,dvipsnames]{xcolor}
\usepackage{mathrsfs} 
\usepackage{url}
\usepackage{enumitem}

 \usepackage{multirow}
\usepackage{pifont}
\usepackage{amssymb}

\usepackage[T1]{fontenc}    % use 8-bit T1 fonts
\usepackage{booktabs}       % professional-quality tables
\usepackage{amsfonts}       % blackboard math symbols
\usepackage{nicefrac}       % compact symbols for 1/2, etc.
\usepackage{microtype}      % microtypography

\usepackage{subcaption}
\usepackage{chngpage}
\usepackage[font=small,labelfont=bf]{caption}
\usepackage{color}
\usepackage{wrapfig}
\usepackage{xspace}
\usepackage{fancyhdr}
\usepackage{natbib}
\usepackage[colorlinks,citecolor=blue,urlcolor=blue,linkcolor=RawSienna]{hyperref}
\usepackage{datetime}

\DeclareMathAlphabet\EuRoman{U}{eur}{m}{n}
\SetMathAlphabet\EuRoman{bold}{U}{eur}{b}{n}

\usepackage[capitalize]{cleveref}

\crefname{lemma}{Lemma}{Lemmas}
\crefname{corollary}{Corollary}{Corollaries}
\crefname{theorem}{Theorem}{Theorems}

\makeatletter
\let\reftagform@=\tagform@
\def\tagform@#1{\maketag@@@{\ignorespaces\textcolor{gray}{(#1)}\unskip\@@italiccorr}}
\renewcommand{\eqref}[1]{\textup{\reftagform@{\ref{#1}}}}
\makeatother

\declaretheorem[style=plain,numberwithin=section,name=Theorem]{theorem}
\declaretheorem[style=plain,sibling=theorem,name=Lemma]{lemma}
\declaretheorem[style=plain,sibling=theorem,name=Proposition]{proposition}

\declaretheorem[style=definition,sibling=theorem,name=Definition]{definition}

\numberwithin{theorem}{section}

%% file: kern-tex/defs.tex
\def\[#1\]{\begin{align}#1\end{align}}
\def\*[#1\]{\begin{align*}#1\end{align*}}

\newcommand{\defas}{\vcentcolon=}  % requires mathtools package
\DeclareMathOperator{\tr}{\mathrm{trace}}

\newcommand{\innerprod}[2]{\langle #1, #2 \rangle}

\newcommand{\rkhs}{\mathcal{H}}

\newcommand{\Reals}{\mathbb{R}}
\newcommand{\Complex}{\mathbb{C}}

\newcommand{\Nats}{\mathbb{N}}
\newcommand{\Ints}{\mathbb{Z}}
\newcommand{\Sphere}{\mathbb{S}}
\newcommand{\Torus}{\mathbb{T}}

\DeclareMathOperator*{\newlim}{\mathrm{lim}\vphantom{\mathrm{infsup}}}
\DeclareMathOperator*{\newmin}{\mathrm{min}\vphantom{\mathrm{infsup}}}

\renewcommand{\lim}{\newlim}
\renewcommand{\min}{\newmin}

\def\bone{\mathbf{1}}

\newcommand{\KL}[2]{\mathrm{KL}\left(#1 || #2\right)}
\newcommand{\bbf}{\mathbf{f}}
\newcommand{\bx}{\mathbf{x}}
\newcommand{\bk}{\mathbf{k}}
\newcommand{\by}{\mathbf{y}}
\newcommand{\bz}{\mathbf{z}}
\newcommand{\bv}{\mathbf{v}}
\newcommand{\bV}{\mathbf{V}}
\newcommand{\bw}{\mathbf{w}}

\newcommand{\bc}{\mathbf{c}}

\newcommand{\bt}{\mathbf{t}}

\newcommand{\bu}{\mathbf{u}}
\newcommand{\be}{\mathbf{e}}

\newcommand{\bs}{\mathbf{s}}

\newcommand{\bX}{\mathbf{X}}

\newcommand{\bS}{\mathbf{S}}

\newcommand{\bI}{\mathbf{I}}

\newcommand{\bK}{\mathbf{K}}
\newcommand{\bM}{\mathbf{M}}

\newcommand{\bR}{\mathbf{R}}
\newcommand{\bC}{\mathbf{C}}
\newcommand{\bU}{\mathbf{U}}

\newcommand{\bZ}{\mathbf{Z}}
\newcommand{\bF}{\mathbf{F}}

\newcommand{\bzero}{\mathbf{0}}

\newcommand{\bmu}{\bm{\mu}}

\newcommand{\bSigma}{\bm{\Sigma}}

\newcommand{\bomega}{\bm{\omega}}

\newcommand{\bLambda}{\bm{\Lambda}}

\newcommand{\elbo}{\mathcal{L}}
\newcommand{\expect}{\mathbb{E}}

\newcommand{\normal}{\mathcal{N}}
\newcommand{\xdomain}{\mathcal{X}}

\newcommand{\loss}{\mathcal{L}}

\newcommand{\bigO}{\mathcal{O}}

\newcommand{\GP}{\mathcal{GP}}

\newcommand{\diag}{\mathrm{diag}}

\newcommand{\vectorize}[1]{\mathrm{vec}(#1)}

%% file: kern-tex/intro.tex
\section{Introduction}

Gaussian Processes (GPs) \citep{rasmussen2006gaussian} are flexible Bayesian nonparametric models which enable principled reasoning about distributions of functions and provide rigorous uncertainty estimates \citep{srinivas2010gaussian, deisenroth2011pilco}. Unfortunately, exact inference in GPs is impractical for large datasets because of the $\bigO(N^3)$ computational cost (for a dataset of size $N$). To overcome the computational roadblocks, sparse Gaussian processes \citep{snelson2006sparse, quinonero2005unifying} use $M$ inducing points to approximate the kernel function, reducing the computational cost to $\bigO(NM^2 + M^3)$. However, %such costs prevent the model from using a large $M$ for accurate kernel approximations, and 
these approaches are prone to overfitting since all inducing points are hyperparameters. 
Sparse variational Gaussian Processes (SVGPs) \citep{titsias2009variational, hensman2015scalable} 
offer an effective protection against overfitting by
framing a posterior approximation using the inducing points and %rigorously 
optimizing them %by maximizing 
with variational inference. 
%Sparse variational Gaussian Processes (SVGPs) \citep{titsias2009variational, hensman2015scalable} frame a variational posterior using the inducing points and rigorously optimize it by maximizing the Evidence Lower Bound (ELBO), which avoids the overfitting issue. 
%Still, the $\bigO(M^3)$ computational costs prevent SVGPs from using large sets of inducing points, thereby limiting the flexibility of variational posteriors. 
Still, the $\bigO(M^3)$ complexity prevents SVGPs from scaling beyond a few thousand inducing points, creating difficulties in improving the quality of approximation.

Several approaches impose structure on the inducing points to increase the approximation capacity. Structured kernel interpolation (SKI) \citep{wilson2015kernel, wilson2015thoughts} approximates the kernel by placing inducing points over a Euclidean grid and exploiting fast structured matrix operations. SKI can use millions of inducing points, but is limited to low-dimensional problems because the grid size grows exponentially with the input dimension.  Other approaches define approximate posteriors using multiple sets of inducing points. \citet{cheng2017variational, salimbeni2018orthogonally} propose to %use two sets of 
decouple the inducing points for modelling means and covariances, %separately, 
leading to a linear complexity with respect to the number of mean inducing points. SOLVE-GP~\citep{shi2020sparse} reformulates a GP as the sum of two orthogonal processes and uses distinct groups of inducing points for each; this has the benefit of 
improving the approximation at a lower cost than standard SVGPs.
%doubling the number of inducing points with only \RBG{Not sure what you meant by this:} 2 times Cholesky decompositions.
%using the Nystrom approximation \citep{drineas2005nystrom}, whose approach doubles inducing points with 2x Cholesky decompositions.
 
%decomposes the GP prior as the summation of a low-rank approximation using inducing points, and a full-rank residual process. Such decomposition naturally leads to using two independent sets of inducing points for each process. As a result, \citet{shi2020sparse} doubles the number of inducing points
%
%Because SVGP enables stochastic optimizations, it is scalable to millions of data points. However, the $\bigO(M^3)$ computational costs still prevent SVGP from using a large set of inducing points, which limits the flexibility of variational posteriors. 

\begin{figure}[t]
\centering
\includegraphics[width=\columnwidth]{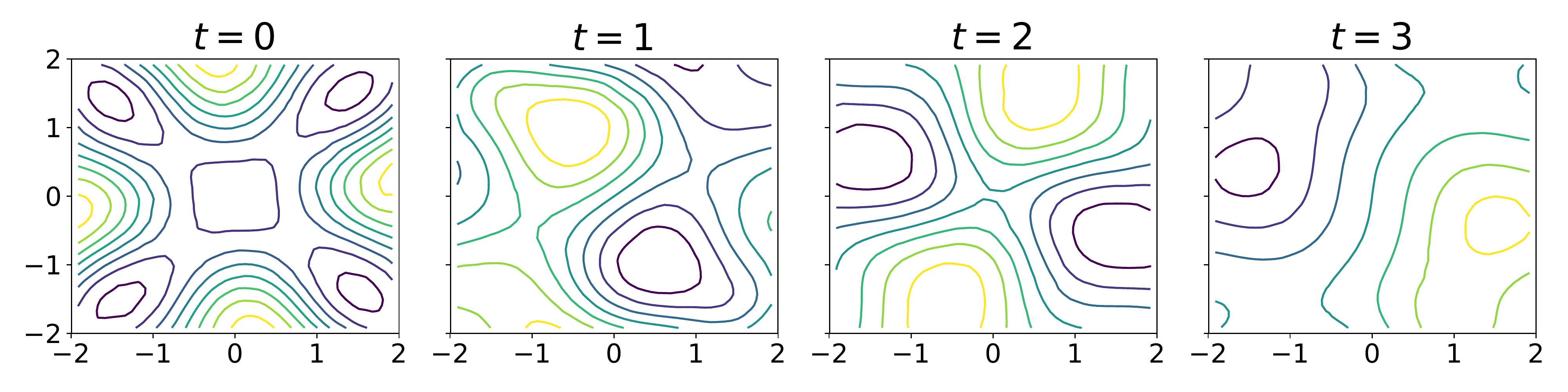}
\caption{Visualizing the harmonic kernel decomposition. 
%Since the $2$-dimensional RBF kernel is invariant to rotations, we decompose it into $k = \sum_{t=0}^3 k_t$ using the symmetry group of $90^\text{o}$ rotations. 
We decompose a $2$-dimensional RBF kernel as $k = \sum_{t=0}^3 k_t$ using the symmetry group of $90^\text{o}$ rotations. 
We plot the real part of random functions sampled from each $\GP(0, k_t)$. Notice that $\GP(0, k_0)$ is invariant to $90^\text{o}$ rotations; $\GP(0, k_1)$ takes opposite values under $180^\text{o}$ rotations; $\GP(0, k_2)$ is invariant to $180^\text{o}$ rotations but has opposite values under $90^\text{o}$ rotations. %Thus each $k_t$ represents a Fourier component of the kernel along the rotations.
\label{fig:rbf-rotate-gp}}
\end{figure} 

%However, in spite of stochastic optimizations, the $\bigO(M^3)$ computational costs still prevent SVGP from using a large set of inducing points, which limits the flexibility of variational posteriors. To lift this obstacle,  \citet{cheng2017variational, salimbeni2018orthogonally} observe that the variational posterior can afford two different sets of inducing points for modelling means and covariances separately. In consequence, they achieve linear complexity with respect to the number of inducing points for means. %Furthermore, \citet{salimbeni2018orthogonally} proposes to include the covariance inducing points in the mean part and combine with natural gradients to remedy the optimization difficulties in \citet{cheng2017variational}. 
%\citet{shi2020sparse} decomposes the GP prior as the summation of a low-rank approximation using inducing points, and a full-rank residual process. Such decomposition naturally leads to using two independent sets of inducing points for each process. As a result, \citet{shi2020sparse} doubles the number of inducing points
%
%\RBG{The reader will ask: what about structured sets of inducing points where inversion is less than $\bigO(N^3)$?}

In this paper, we introduce a more scalable variational approximation for GPs via the proposed \emph{harmonic kernel decomposition (HKD)}, which decomposes the kernel as a sum of orthogonal kernels, $k(\bx, \bx') = \sum_{t=0}^{T-1} k_t(\bx, \bx')$, %by exploiting the input-space symmetries 
using Fourier series.
The HKD reformulates the Gaussian process into an additive GP, %$\GP(0, k) = \sum_{t=0}^{T-1} \GP(0, k_t)$, 
where each %addition GP 
subprocess 
%represents a spectral harmonic \RBG{Are you sure ``spectral harmonic'' is a term?} 
models a Fourier component of the function (see Figure~\ref{fig:rbf-rotate-gp} for a visualization). 
We then propose the Harmonic Variational Gaussian Process~(HVGP), which uses separate sets of inducing points for the subprocesses.
%variational posterior independently for each addition GP. 
%Because of the kernel decomposition and the posterior independence, 
Compared to a standard variational approximation, HVGPs allow us to use a large number of inducing points at a much lower computational cost.
Moreover, HVGPs have an advantage over SKI in that they allow trainable inducing points. % the curse of dimensionality compared to SKI.
Finally, unlike the SOLVE-GP whose decomposition involves only two subprocesses, our HVGP is easily applicable to multiple subprocesses and can be computed efficiently with the discrete Fourier transform.

%Through a range of problems and GP models, we demonstrate that HVGP is scalable to large number of inducing points and applicable to general mode

Empirically, we demonstrate the scalability and general applicability of HVGPs through a range of problems and models including using RBF kernels for modelling earth elevations and using convolutional kernels for image classification. In these experiments, HVGPs significantly outperform standard variational methods by exploiting the input-space symmetries such as translations and reflections. Our model can further exploit parallelism to achieve minimal wall-clock overhead when using 8 groups of inducing points.
%In CIFAR-10 classification using deep convolutional Gaussian processes, we show that our method works for complex GP models \RBG{What do you mean by ``complex''?}, where it allows 4 times inducing points and achieves the highest accuracies and test log likelihoods. \RBG{The abstract claims SOTA results. Do we want to be more specific here?}
In CIFAR-10 classification, %using deep convolutional GPs, 
we show that our method %works for 
can be integrated with deep convolutional structures to achieve state-of-the-art results for GPs.

%% file: kern-tex/background.tex
\section{Background}
\label{sec:bg}
%In this section we briefly review Fourier series and Gaussian processes.
%, which build the background for the proposed harmonic kernel decompositions and scalable variational GP models in the next sections. To start with, we define the notations used throughout the paper. 

%\textbf{Notations.} We use lowercase characters $a$, bold lowercase characters $\ba$, and bold uppercase characters $\bA$ to represent scalars, vectors, and matrices, respectively. We use $\Reals$ and $\Complex$ to represent real and complex numbers, respectively. We use $\bA^{\top}$ and $\bA^H$ to represent transposes and conjugate transposes, respectively. 
%\sjx{I think all these follow the convention. Do we need to explicitly mention them?}

\subsection{Discrete Fourier Transform}
\label{sec:bg-dft}

Fourier analysis \citep{baron1878analytical} studies the representation of functions as sums or integrals of sinusoids.
%concerned in representing the functions as sums of harmonic series. 
For an integrable function $f$ on $\Reals^d$, its Fourier transform is defined as
\begin{align}
    \hat{f}(\bomega) = \int_{\mathbb{R}^d} e^{-2\pi i \bomega^\top\bx} f(\bx)\, \mathrm{d}\bx,
\end{align}
%returns a function $\hat{f}$ of the frequency $\bomega$
where $\hat{f}(\bomega) \in \Complex$.
%, which characterizes the frequency $\bomega$ present in $f$. % The inverse Fourier transform maps $\hat{f}$ back to $f$, $f(\bx)  = \int_{\mathbb{R}^d} \exp(2\pi i \bomega^\top\bx) \hat{f}(\bomega) \mathrm{d}\bomega$. 
In kernel theory, Bochner's Theorem \citep{bochner1959lectures} is a seminal result that uses the Fourier transform to establish a bijection between stationary kernels and positive measures in the spectral domain. 

Fourier analysis can be performed over finite sequences as well, via the discrete Fourier transform (DFT) \citep{cooley1969finite}. Specifically, given a sequence $\bx = [\bx_0, ..., \bx_{T-1}]^{\top}$\!, DFT computes the sequence %of Fourier series 
$\hat{\bx} = [\hat{\bx}_0, ..., \hat{\bx}_{T-1}]^{\top}$\!,  with
\begin{align}
   \hat{\bx}_{t} = \frac{1}{T}\sum_{s=0}^{T-1} \bx_s e^{-i\frac{2\pi t s}{T}}, \; t=0,...,T-1.
\end{align}
%DFT can be represented in vector form as %well:
%$\hat{\bx} = \bF \bx$, where 
Let $\bF:=[\frac{1}{T} e^{-i\frac{2\pi t s}{T} } ]_{t,s=0}^{T-1} \in \Complex^{T \times T}$ denote the DFT matrix.
%\begin{align}\label{eq:dft-mat}
%    \bF = [\frac{1}{T} e^{-i\frac{2\pi t s}{T} } ]_{t,s=0}^{T-1}, 
%\end{align}
The DFT can be represented in vector form as %well:
$\hat{\bx} = \bF \bx$, which naturally leads to the inverse DFT: $\bx=\bF^{-1}\hat{\bx}$.
%Then the discrete Fourier transform can be written compactly as : $\hat{\bx} = \bF \bx$. % We can also compute the inverse discrete Fourier transform: $\bx = \bF^{-1}\hat{\bx}$, where $\bF^{-1} = [e^{i\frac{2\pi nm}{N}}]_{n,m=0}^{N-1}$. 

More generally, if $\bX \in \Complex^{T_1 \times \cdots \times T_k}$ is a tensor, the multidimensional DFT computes the tensor $\hat{\bX} \in \Complex^{T_1 \times \cdots \times T_k}$,
\begin{align}
    \hat{\bX}[t_1, ..., t_k] = \frac{1}{T} \sum_{s_1, ..., s_k} X[s_1, ..., s_k] \prod_{j=1}^k e^{-i\frac{2\pi  t_j s_j}{T_j}}, \notag 
\end{align}
where $T = \prod_{j}T_j$. The DFT matrix is then a tensor product of one-dimensional DFT matrices.

\subsection{Gaussian Processes}
Given an input domain $\xdomain$, a mean function $m$, and a kernel function $k: \xdomain \times \xdomain \to \Reals$, the Gaussian process \citep{rasmussen2006gaussian} $\GP(m, k)$ is a distribution %$p(f)$ 
over functions $\xdomain \to \Reals$. %Although our approach is applicable to any mean function, 
For any finite set $\{\bx_1,\bx_2, \dots, \bx_N\} \subset \xdomain$, the function values $\bbf = [f(\bx_1), f(\bx_2), ..., f(\bx_N)]^\top$ have a multivariate Gaussian distribution: %Gaussian processes are parameterized by a covariance function or kernel function $k(\cdot, \cdot)$. The marginal distribution of function values is given by
\begin{equation}
\bbf \sim \mathcal{N}(m(\bX), \mathbf{K}_{\bbf \bbf}), \notag 
\end{equation}
where $m(\bX) =[m(\bx_1), ..., m(\bx_N)]^\top$\!, and $\mathbf{K}_{\bbf \bbf} = [k(\bx_i,\bx_j)]_{i,j=1}^N$. 
For simplicity we assume $m(\cdot) = 0$ throughout the paper. 
The observations $y$ are modeled with a density $p(y|f(\bx))$, often taken to be Gaussian in the regression setting: $y = f(\bx) + \epsilon, \quad \epsilon \sim \mathcal{N}(0, \sigma^2)$.
%For regression problems, the likelihood usually models independent Gaussian noise: $y = f(\bx) + \epsilon, \quad \epsilon \sim \mathcal{N}(0, \sigma^2)$. 
Let  $(\bX, \by)$ be a training set of size $N$. The posterior distribution $p(\bbf^{\star}|\by)$ under a Gaussian observation model is 
%is an explicit Gaussian 
\begin{align}
\normal( \bK_{\star \bbf}(\bK_{\bbf \bbf} + \sigma^2\bI)^{-1}\by, \bK_{\star \star} - \bK_{\star \bbf}(\bK_{\bbf \bbf} + \sigma^2\bI)^{-1}  \bK_{\bbf \star}), \notag
\end{align}
%$\tilde{\bK}_{\bbf \bbf} = \bK_{\bbf \bbf} + \sigma^2\bI_N$, the posterior is an explicit Gaussian under Gaussian likelihoods,
%\begin{align}
%\bbf^{\star}|\by \sim \normal( \bK_{\star \bbf}\tilde{\bK}_{\bbf \bbf}^{-1}\by, \bK_{\star \star} - \bK_{\star \bbf}\tilde{\bK}_{\bbf \bbf}^{-1}  \bK_{\bbf \star}), \notag
%\end{align}
where $\bbf^{\star}=f(\bX^{\star})$ are function values at test locations. Unfortunately, computing the posterior mean and covariance requires inverting the kernel matrix, an $\bigO(N^3)$ computation.

Sparse variational Gaussian processes (SVGPs) \citep{titsias2009variational, hensman2013gaussian} use inducing points for scalable GP inference. Let $\bZ = [\bz_1, ..., \bz_M]^{\top} \in \Reals^{M \times d}$ be $M$ inducing locations, and $\bu = f(\bZ)$. 
SVGPs consider an augmented joint likelihood, $p(f(\cdot), \bu) = p(f(\cdot)| \bu) p(\bu)$, and a variational approximation $q(f(\cdot), \bu) = p(f(\cdot) | \bu) q(\bu)$, where $ q(\bu)=\normal(\bmu, \bS)$ is a parameterized multivariate Gaussian with mean $\bmu$ and covariance $\bS$. 
The variational approximation is optimized by maximizing the variational lower bound:
%Evidence Lower Bound (ELBO) 
%with respect to $\bmu, \bS$:
\begin{align}
    \elbo \defas \expect_{q(\bbf, \bu)} [\log p(\by|\bbf, \bX)] - \KL{q(\bu)}{p(\bu)}.
\end{align}
Since $\log p(\by|\bbf, \bX) = \sum_{i=1}^N \log p(\by_i|f(\bx_i))$ 
%, the ELBO 
admits stochastic optimization,
%In consequence, 
SVGPs reduce the computational cost to $\bigO(M^3 + M^2B)$, where $B$ is the minibatch size.

%% file: kern-tex/kernel.tex
\section{Harmonic Kernel Decomposition}\label{sec:kern-decomp}
In this section, we introduce kernel Fourier series and use them %for the dea
to form the harmonic kernel decomposition. 
All proofs can be found in \Cref{app:proof:ortho}. % in the appendix.

\subsection{Kernel Fourier Series}

We first propose a general method for representing a kernel as a sum of functions. 
The idea is based on the DFT (see Sec.~\ref{sec:bg-dft}). Let $k: \mathcal{X}\times\mathcal{X} \to \Complex$ be a positive definite kernel.
To apply the DFT, we fix the first input $\bx$, and construct a finite sequence of kernel values using a transformation $G: \mathcal{X}\to\mathcal{X}$ that applies to the second input:
%in the input space:
\begin{equation} \label{eq:dft-seq}
	[k(\bx, G^0(\bx')), k(\bx, G^1(\bx')), ..., k(\bx, G^{T-1}(\bx'))],
\end{equation}
where $G^0(\bx) := \bx$ and $G^t := G \circ G^{t - 1}$.
Note that in signal processing, the sequence under DFT usually contains equally-spaced samples along the time domain. Here, we adopt a more general form by using $G$ to exploit symmetries in the input domain.

%We introduce the kernel Fourier series, which conducts Fourier analysis on kernels. We first define the $T$-cyclic transformation on whose orbits the DFT will be performed.

\begin{definition}[Kernel Fourier Series]
We define $T$ complex-valued functions $k_t: \xdomain \times \xdomain \to \Complex$, $t=0, ~\dots,~T-1$ using the DFT of the sequence \eqref{eq:dft-seq}: 
% Specifically, the $kt$-th function represents the $t$-th frequency harmonics from the DFT,
\begin{align} \label{eq:def-kt}
	k_t(\bx, \bx') =  \sum_{s=0}^{T-1} \bF_{t, s} k(\bx, G^s(\bx')),
\end{align}
where $\bF$ is the DFT matrix (see Sec.~\ref{sec:bg-dft}). 
The \emph{kernel Fourier series} of $k(\bx, G^s(\bx'))$ is given by the inverse DFT:
%Then kernel Fourier series write $k(\bx, G^s(\bx'))$ as a weighted sum of $k_t$ by the inverse DFT:
\begin{align} \label{eq:kernel-fourier-series}
	k(\bx, G^s(\bx')) = \sum_{t=0}^{T-1} \bF_{s,t}^{-1} k_t(\bx, \bx'). 
\end{align}
The inverse DFT matrix is $\bF^{-1} = T\bF^H$, where $\cdot^H$ is the conjugate transpose.
\end{definition}
%\textbf{Kernel Fourier Series.} 

Given the positive definiteness of $k$, it is tempting to ask whether $k_t$ is also a (complex-valued) kernel. 
In the next section, we will study the conditions when this holds and use it to form an orthogonal kernel decomposition.

\subsection{Harmonic Kernel Decomposition}\label{subsec:kernel-decom}
%Now we demonstrate that the kernel Fourier series form an orthogonal kernel sum decomposition of the kernel $k$.
We first introduce the following definitions.

\begin{definition}[$T$-Cyclic Transformation]
A function $G: \xdomain \to \xdomain$ is $T$-cyclic if $T$ is the smallest integer such that,
\begin{align}
	\forall \bx \in \xdomain,\;G^T (\bx) \defas \overbrace{G \circ \cdots \circ G}^{T} (\bx) = \bx.
\end{align}
%or in other words, $T$ consecutive applications of $G$ is the identity operator.
\end{definition}
In group theory, $\{G^0, G^1, \dots, G^{T - 1}\}$ forms a cyclic group of order $T$, and $G$ is the generator of this group. 
Interestingly, given a $T$-cyclic $G$, multiplying $k_t$ with $e^{i\frac{2\pi t}{T}}$ corresponds to a shift by $G$ in the second input. 
%can prove the following proposition.
\begin{proposition}[Shift]\label{prop:invariance} For any $0\leq t \leq T-1$,
\begin{align}
    k_t(\bx, G(\bx')) = e^{i\frac{2\pi t}{T}} k_t(\bx, \bx'). 
\end{align}
\end{proposition}
%We observe that $k_t$ multiplies with $e^{i\frac{2\pi t}{T}}$ when $x'$ is transformed to $G(x')$. 
%In particular, $k_0$ is constant on the transformation orbit, which is similarly shown in Figure~\ref{fig:rbf-rotate-gp}.
From \Cref{prop:invariance} we have $k_0(\bx, G(\bx')) = k_0(\bx, \bx')$, and when $T$ is even, $k_{T/2}(\bx, G^2(\bx')) = k_{T/2}(\bx, \bx')$. 
This property is illustrated in \Cref{fig:rbf-rotate-gp}. 

\begin{definition}[$G$-Invariant kernels]
A kernel function $k: \xdomain \times \xdomain \to \Complex$ is $G$-invariant if,
\begin{align}
	\forall \bx, \bx' \in \xdomain,\; k(G(\bx), G(\bx')) = k(\bx, \bx').
\end{align}
\end{definition}

For example, a polynomial kernel $k(\bx, \bx')=(\bx^{\top}\bx'+c)^t$ is invariant to the rotation transformations.

\begin{theorem}[Harmonic Kernel Decomposition] \label{thm:kern-decomp}
Let $G$ be a $T$-cyclic transformation, and $k$ be a $G$-invariant kernel. Then, the following decomposition holds:
\begin{align}
    k(\bx, \bx') = \sum_{t=0}^{T-1} k_t(\bx, \bx'),
\end{align}
where $k_t,\;t=0, \dots,T-1$ are defined as in Eq.~\eqref{eq:def-kt}. Moreover, for any $0\leq t \leq T-1$, $k_t$ is a Hermitian kernel.
\end{theorem}
The equation follows from the kernel Fourier series of $k(\bx, G^0(\bx'))$ by noticing that $\bF^{-1}_{0,:} = \bone$. To prove that $k_t$ is a kernel, we show that $k_t(\bx, \bx') = \bF_{:, t}^H \bK(\bx, \bx') \bF_{:, t}$, where $\bK(\bx, \bx') = [k(G^{s_1}(\bx), G^{s_2}(\bx'))]_{s_1, s_2=0}^{T-1}$,

Besides the kernel sum decomposition, we further show that the kernels $k_t,\;t=0,...,T-1$ are orthogonal to each other, as identified by the following lemma.

\begin{lemma}[Orthogonality]\label{lem:ortho} For any $0\leq t_1 \neq t_2 \leq T-1$, let $\rkhs_k, \rkhs_{k_{t_1}}, \rkhs_{k_{t_2}}$ be the RKHSs corresponding to the kernel $k, k_{t_1}, k_{t_2}$, respectively. Then for any  $f \in \rkhs_{k_{t_1}}$ and $g \in \rkhs_{k_{t_2}}$, $\innerprod{f}{g}_{\rkhs_k} = 0$.
\end{lemma}
Because the $\rkhs_k$ inner product of $f \in \rkhs_{1}, g \in \rkhs_{2}$ is always zero, we immediately obtain that the RKHSs for $k_t$ are disjoint except for the function $f \equiv 0$.
\begin{proposition}[Disjoint]\label{prop:disjoint} For any $0\leq t_1 \neq t_2 \leq T-1$,
\begin{align}
    \rkhs_{k_{t_1}} \cap \rkhs_{k_{t_2}} = \{0\},
\end{align}
\end{proposition}

The kernel decomposition and orthogonality translate to the RKHS orthogonal sum decomposition as follows:
\begin{theorem}[Orthogonal Sum Decomposition of RKHS]\label{thm:rkhs-decom} The RKHS $\rkhs_k$ admits an orthogonal sum decomposition,
\begin{align}
    \rkhs_k = \bigoplus_{t=0}^{T-1} \rkhs_{k_t}.
\end{align}
Specifically, for any function $f \in \rkhs_k $, $f$ has the unique decomposition $f = \sum_{t=0}^{T-1} f_t,\; f_t \in \rkhs_{k_t}$, and
%\begin{align}
$f_t(\bx) = \innerprod{f}{k_t(\bx, \cdot)}_{\rkhs_k}$.
%\end{align}
The RKHS norm of $f$ is equal to
\begin{align}
    \|f\|^2_{\rkhs_k} = \sum_{t=0}^{T-1} \|f_t\|^2_{\rkhs_{k_t}}. 
\end{align}
\end{theorem}

\subsection{Examples of Harmonic Kernel Decomposition}\label{subsec:example}
The HKD relies on the $(G, k)$ pair where $G$ is a $T$-cyclic transformation and $k$ is a kernel invariant to $G$. 
In this section we provide examples of such kernels and transformations. 
Notably, all inner-product kernels and stationary kernels\footnote{This includes, e.g., polynomial, Gaussian, Mat\'ern, periodic, arccosine, and rational quadratic kernels.} can be decomposed with the HKD when paired with an appropriately chosen $G$.
%Now we show several examples to illustrate how the kernel harmonics are decomposed. % , which are also summarized in Table~\ref{tab:k-g-pair}.

%\begin{table*}[t]
%\centering
%\begin{tabular}{cccc}
%\toprule
%Kernels                   & Inner-Product          & Stationary & Stationary    \\ 
%Input Space               & $\Complex^d, \Reals^d$ & $\Reals^d$ & $\Torus^d$  \\
%Transformations & Rotation, Reflection  & Negation   & Translation \\
%\bottomrule
%\end{tabular}
%\caption{The cyclic transformations to which the kernel is invariant.\label{tab:k-g-pair}}
%\end{table*}

\paragraph{An opening example.} We start with a toy example to illustrate the kernel decomposition. Let $k(\theta, \theta') = e^{-i(\theta - \theta')} + e^{-2i(\theta - \theta')}$ for $\theta \in [0, 2\pi)$. The transformation $G(\theta) = (\theta + \frac{2\pi}{T}) \mod 2\pi$ is  $T$-cyclic. Based on the kernel Fourier series, we obtain $k_{1}(\theta, \theta') = e^{-i(\theta - \theta')}$, $k_{2}(\theta, \theta') = e^{-2i(\theta - \theta')}$, and $k_t = 0$ otherwise. We observe that the RKHS of $k_2$ contains periodic functions with basic period $\pi$, while the RKHS of $k_1$ contains functions with basic period $2\pi$. %but not of period $\pi$. 
In this way, our method decomposes $\rkhs_k$ into orthogonal RKHSs. %based on harmonics.

\textbf{Inner-product kernels} are kernels of the form,
\begin{align}
    k(\bx, \bx') = h(\bx^H\bx, \bx^H\bx', \bx'^H\bx'),
\end{align}
where the function $h$ ensures that $k$ is positive semi-definite \citep{hofmann2008kernel}. For a matrix $\bR \in \Complex^{d \times d}$, which is unitary (i.e.~$\bR \bR^{H} = \bI$), 
%\RBG{check that I put the parentheses in the right place} 
the kernel $k$ is $G$-invariant:
\begin{align}
    \!k(G(\bx), G(\bx'))\!&=\!h(\bx^H \bR^H \bR \bx, \bx^H \bR^H \bR \bx', \bx'^H \bR^H \bR  \bx') \notag \\
    \!&=\! h(\bx^H\bx, \bx^H\bx', \bx'^H\bx') = k(\bx, \bx') \notag,
\end{align}
Examples include reflections, rotations, and permutations. 
Moreover, if $\underbrace{\bR \cdots \bR}_T = \bI$, 
the mapping $G(\bx) = \bR \bx$ is  $T$-cyclic. 
%because $\bR$ is unitary, 
%Therefore, the kernel Fourier series $k_t$ captures the $t$-th harmonic of the kernel $k$ along the orbit defined by $\bR$.

\textbf{Stationary kernels} are kernels of the form,
\begin{align}
    k(\bx, \bx') = \kappa(\bx - \bx'),
\end{align}
where $\kappa$ is a positive-type function \citep{berlinet2011reproducing}. Let $T=2$ and $G(\bx) = -\bx$; then $G$ is $T$-cyclic. For real kernels whose $k(\bx, \bx') \in \Reals$, the kernel is symmetric (i.e.~$k(\bx, \bx')=k(\bx', \bx)$). Then $k$ is $G$-invariant:
\begin{align}
    k(G(\bx), G(\bx')) = \kappa(\bx' - \bx) = k(\bx', \bx) = k(\bx, \bx'). \notag 
\end{align} 
Similarly, we can prove that inner-product kernels are negation-invariant. 
%\textbf{Stationary kernels on a multi-dimensional torus} are defined over 
Stationary kernels are also invariant to the $T$-cyclic transformation $G_i(\bx) = \bx + \frac{2\pi}{T} \be_i$ on a multidimensional torus
%the torus 
$\Torus^d = \underbrace{\Sphere^1 \times \cdots \times \Sphere^1}_d$, where $\Sphere^1$ represents a one-dimensional circle. 

\subsection{Resolving Complex-Valued Kernels}\label{subsec:resolve-hermitian}

From $\bF = [\frac{1}{T} e^{-i\frac{2\pi t s}{T} } ]_{t,s=0}^{T-1}$, we know that the DFT introduces complex values whenever $T > 2$. 
%Harmonic kernel decomposition arises from the discrete Fourier transform, 
Therefore, $k_t$ is Hermitian but not necessarily real-valued. 
For example, the decomposition in Fig.~\ref{fig:rbf-rotate-gp} introduces imaginary values when $t = 1, 3$.
Since $k$ is real-valued, we can obtain a real-valued kernel decomposition by pairing up $k_t$s.
Specifically, for a $T$-cyclic transformation $G$, we have
\begin{align}
	(k_t + k_{T-t})(\bx, \bx') = \frac{2}{T}\sum_{s=0}^{T-1} \cos(\frac{2\pi t s}{T}) k(\bx, G^s(\bx')), \notag
\end{align}
%Consider a single-way transformation with period $T$, we sum up $k_t$ and $k_{T-t}$ together,
%\begin{align}
%	(k_t + k_{T-t})(\bx, \bx') = \sum_{s=0}^{T-1} \left(\bF_{t, s}+\bF_{T-t, s} \right) k(\bx, G^s(\bx')), \notag
%\end{align}
%Because $\bF_{t, s}+\bF_{T-t, s} = \frac{2}{T} \cos(\frac{2\pi t s}{T})$, $k_t + k_{T-t}$ is a real-valued kernel. 
%By pairing up all Hermitian kernels, 
In this way we obtain a real-valued decomposition with $\lfloor T/2 \rfloor+1$ kernels.
%
%\begin{align}
%    k(\bx, \bx') = \sum_{t=0}^{\lfloor T/2 \rfloor+1} \left(k_t + k_{T-t}\right)(\bx, \bx') .
%\end{align}
%Similarly, for $J$-way transformations with periods $(T_1, ..., T_J)$, we obtain a real-valued decomposition with $(\lfloor T_1/2 \rfloor+1) \times \cdots \times (\lfloor T_J/2 \rfloor+1)$ kernels.

\subsection{Multi-Way Transformations}\label{subsec:multi-dim}

%\RBG{It's not obvious to me that putting the multidimensional case in a separate section helps clarity. In the preceding discussion, do we gain anything by restricting to the cyclic case?  In fact, the generalized DFT (i.e.~the generalization to arbitrary groups) might even be notationally simpler, and we can point out cyclic groups and product groups as special cases.} \ssy{we are restricted to cyclic groups, otherwise the decomposed $k_t$ might not be a kernel.} 
Previously we considered the Fourier series along one transformation orbit: $k(\bx, G^0(\bx')), ..., k(\bx, G^{T-1}(\bx'))$. We can extend it to multi-way transformations, akin to a multidimensional DFT. Let $T_1, ..., T_J \in \Nats$ and $G_j$ be a $T_j$-cyclic transformation for $j=1,...,J$, respectively. We further assume that all transformations commute, i.e., $\forall 1\leq j_1, j_2 \leq J$, 
\begin{align}
\forall \bx \in \xdomain, G_{j_1} (G_{j_2} (\bx)) =  G_{j_2} (G_{j_1} (\bx)).
\end{align}
Due to commutativity, we can use the indices $(t_1, ..., t_J)$ to represent applying each $G_j$ for $t_j$ times, 
\begin{align}
    G^{(t_1, ..., t_J)}(\bx) \defas G_1^{t_1}\cdots G_J^{t_J}(\bx),
\end{align}
where $G:= G_1 \otimes \cdots \otimes G_J$. Moreover, if a kernel $k$ is $G_j$-invariant for all $j=1,...,J$, then $k$ is $G$-invariant. 

%We define the kernel Fourier series correspondingly. 
Letting $t=(t_1, ..., t_J)$ be a multi-index, we compute the $J$-way kernel Fourier series from a multidimensional DFT:
%\begin{align}
%    &k_{(t_1, ..., t_J)}(\bx, \bx') \notag \\
%    &= \sum_{s_1=0}^{T_1 - 1} \cdots \sum_{s_J=0}^{T_J - 1} \prod_{j=1}^J \bF^{(j)}_{t_j, s_j} k(\bx, G^{(s_1, ..., s_J)}(\bx')), \notag
%\end{align}
\begin{align}
    k_{t}(\bx, \bx') 
    = \sum_{s=(0, ..., 0)}^{(T_1 - 1, ..., T_J-1)} \prod_{j=1}^J \bF^{(j)}_{t_j, s_j} k(\bx, G^{s}(\bx')), \notag
\end{align}
where $\bF^{(j)} \in \mathbb{C}^{T_j \times T_j}$ is the DFT matrix of %length $T_j$ for 
the $j$-th transformation. 
Similar to \Cref{thm:kern-decomp}, these $k_t$s also form an HKD of $k$. 
%The $J$-way kernel Fourier series form the HKD as well.

Taking a two-dimensional RBF kernel as an example, we can check that it is invariant to negation along either dimension: $G_1([x_1, x_2]^{\top}) = [-x_1, x_2]^{\top}, G_2([x_1, x_2]^{\top}) = [x_1, -x_2]^{\top} $. Because $G_1$ and $G_2$ commute, this forms a 2-way transformation $G = G_1 \otimes G_2$, which corresponds to an HKD with $2 \times 2 = 4$ sub-kernels.

%% file: kern-tex/spectral-decoupled-inference.tex
\section{Harmonic Variational Gaussian Processes}\label{sec:hvgp}
In this section, we explore the implications of the HKD, and propose a scalable inference strategy for variational Gaussian processes. All proofs can be found in Sec~\ref{app:prof:lemma-dvi} in the appendix.

\subsection{Variational Inference for Decomposed GPs }\label{subsec:hf-gp}
Given the kernel decomposition\footnote{$t$ can be a multi-index for multi-way transformations.} $k = \sum_{t=0}^{T-1} k_t$, the Gaussian process can be represented in an additive formulation,
\begin{align}\label{eq:additive-gp}
    f = \sum_{t=0}^{T-1} f_t, \quad f_t \sim \GP(0, k_t). % t=0,...,T-1,
\end{align}
For $t=0,...,T-1$, we introduce inducing points $\bZ_t$ and denote by $\bu_t := f_t(\bZ_t)$ the inducing variables. Let $p_t$ represent $\GP(0, k_t)$. We consider an augmented model,
\begin{align}\label{eq:joint-ll}
    f &= \sum_{t=0}^{T-1} f_t, \quad
    p_t(f_t(\cdot), \bu_t) = p_t(f_t(\cdot) | \bu_t) p_t(\bu_t), % t=0,...,T-1, 
%    \notag 
\end{align}
%Then a variational posterior in harmonic formulation is,
and define the variational posterior approximation as
\begin{equation}
\begin{aligned} \label{eq:full-q-h}
    f &= \sum_{t=0}^{T-1} f_t, \\
    f_t(\cdot) &\sim p_t(f_t(\cdot) | \bu_t), \; \bu_{0:T-1} \sim q(\bu_{0:T-1}). % t=0,...,T-1, 
%    \notag 
\end{aligned}
\end{equation}
%Now a natural question to ask is,
To understand how well this variational distribution approximates the true GP posterior, 
%We answer it by comparing 
we compare it with a standard SVGP, for which %\sjx{Should we remove this statement? Burt et al. only studies the collapsed bound} 
the quality of approximation has been studied extensively by \citet{burt2019rates}. 
For simplicity, we focus our analysis on the case where inducing points are shared across all subprocesses: $\bZ_0 = ... = \bZ_{T-1} := \bZ$, and we assume a complex-valued kernel decomposition without using the techniques in \Cref{subsec:resolve-hermitian}.
%assume we deal with the complex-valued kernel decomposition 
%GPs without resolving Hermitian kernels. 
%Moreover, we 
%and we share the inducing points across all subprocesses: $\bZ_0 =  \cdots \bZ_{T-1} := \bZ$. 
Then, we demonstrate that Eq.~\eqref{eq:full-q-h}
%the variational posterior 
is equivalent to an SVGP with inducing points $\{G^t(\bZ)\}_{t=0}^{T-1}$.
\begin{theorem}\label{lem:q-equi} Consider an SVGP with inducing points $\{G^t(\bZ)\}_{t=0}^{T-1}$. 
Let $\bv_t := f(G^t(\bZ))$ be the inducing variables and $\bV:=[\bv_0, ..., \bv_{T-1}]^{\top} \in \Complex^{T \times m}$. 
Suppose its variational distribution is
\begin{align}
    q_{\mathrm{svgp}}(f(\cdot), \bV) &= p(f(\cdot)|\bV) \normal(\vectorize{\bV}| \vectorize{\bM_{v}}, \bS_v),  \notag 
\end{align}
where $\bM_v \in \Complex^{T \times m}$, $\bS_v \in \Complex^{Tm \times Tm}$ are the mean and covariance, respectively. 
Let $\bU:=[\bu_0, ..., \bu_{T-1}]^{\top} \in \Complex^{T \times m}$. 
Then, Eq.~\eqref{eq:full-q-h} %using $q$ 
and %the SVGP posterior 
$q_{\mathrm{svgp}}$ have the same marginal distribution of $f(\cdot)$ if $q(\bu_{0:T-1})$ is defined as
\begin{align}
    q(\vectorize{\bU}) = \normal(\vectorize{\bF^H \bM_v}, (\bI \otimes \bF^H) \bS_v (\bI \otimes \bF)). \notag 
\end{align}
%where $\bF \in \Complex^{T \times T}$ is the DFT matrix.
 %the variational posterior in 
\end{theorem}
The proof is based on showing the bijective linearity $\bU = \bF^H \bV$. 
Since the theorem assumes shared inducing points, our variational approximation in Eq.~\eqref{eq:full-q-h} %approximates the true posterior at least as well as 
has a larger capacity than
SVGPs with inducing points $\{G^t(\bZ)\}_{t=0}^{T-1}$. 
Therefore, if the inducing points $\{G^t(\bZ)\}_{t=0}^{T-1}$ match the input distribution well, our variational posterior can approximate the true posterior accurately.
%approximation can be accurate.

%And the SVGP approximation accuracy depends on the how well $\{G^t(\bZ)\}_{t=0}^{T-1}$ spread the input space.

\subsection{Harmonic Variational Gaussian Processes}\label{subsec:hvgp}

\begin{figure}[t]
\centering
\includegraphics[width=\columnwidth]{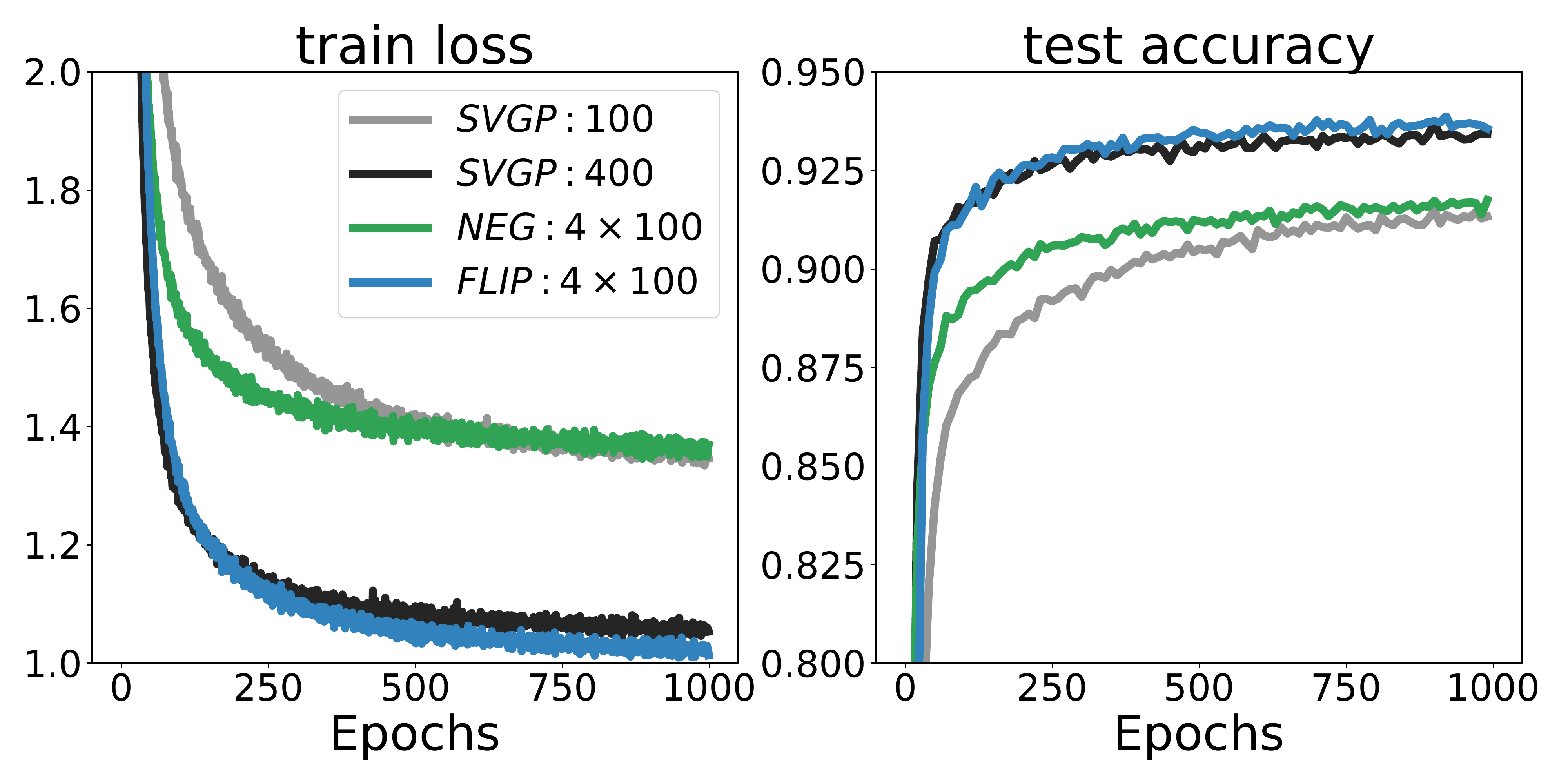}
\caption{Flip-MNIST.  We plot how \textit{left}: train loss and \textit{right}: test accuracy evolve with training. We observe the HVGP using \textit{FLIP: $4 \times 100$} perform similarly with the SVGP using \textit{$400$} inducing points while the HVGP using \textit{NEG: $4 \times 100$} perform substantially worse. \label{fig:flipmnist}}
\end{figure} 

% The variational posterior in Eq.~\eqref{eq:full-q-h} requires a joint distribution $q(\bu_{0:T-1})$. 
The additive GP reformulation in Eq.~\eqref{eq:additive-gp} ensures the independence between $\bu_{0:T-1}$ in the prior, and we demonstrated the orthogonality of the decomposed RKHSs in Theorem~\ref{thm:rkhs-decom}. Thus it is tempting to modelling the variational posterior separately within each RKHS. Now we introduce the Harmonic Variational Gaussian Process (HVGPs), which enforces independence between $\bu_t$ by letting $q(\bu_{0:T-1}) = \prod_{t=0}^{T-1} q_t(\bu_t)$. Then the variational posterior becomes
\begin{align} \label{eq:diag-q-h}
    f &= \sum_{t=0}^{T-1} f_t, \;
    q_t(f_t(\cdot), \bu_t) = p_t(f_t(\cdot) | \bu_t) q_t(\bu_t). % t=0,...,T-1, 
\end{align}
In other words, HVGPs use a variational posterior independently for each GP. We set $q_t(\bu_t) = \normal(\bmu_t, \bS_t)$ as Gaussians. In particular, if each $q_t$ uses $m$ inducing points, we term the model a $T\times m$ model. The variational posterior can be optimized by maximizing the ELBO,
\begin{align*}
    %\elbo \defas 
    \resizebox{\columnwidth}{!}{$\displaystyle
    \expect_{q(\{f_t\}^{T-1}_{t=0})} [\log p(\by|\sum_{t=0}^{T-1}f_t, \bX)]  - \sum_{t=0}^{T-1} \KL{q_t(\bu_t)}{p_t(\bu_t)}.$}
\end{align*}
where $q(\{f_t(\cdot)\}^{T-1}_{t=0}) := \prod_{t=0}^{T - 1} \int p_t(f_t(\cdot)|\bu_t)q_t(\bu_t) d\bu_t$.

%A natural question to ask is 
How well does $\prod_{t=0}^{T-1} q_t(\bu_t)$  approximate the optimal Gaussian variational posterior $q^{\star}(\bu_{0:T-1})$? %in harmonic formulation? 
%Under Gaussian likelihoods, 
$q^{\star}$ has a covariance $\bS^{\star} \in \Reals^{Tm \times Tm}$, while HVGPs induce block diagonal structures. 
Fortunately, we can show that $\bS^{\star}$ is approximately block diagonal if the input distribution is invariant to $G$.

\begin{theorem}\label{lem:decoupled} If the input distribution $p$ is invariant to $G$, i.e., the random variable $\bx \sim p$ and the random variable $G(\bx), \bx \sim p$ are identically distributed, then %for Gaussian likelihoods, 
$\bS^{\star}$ becomes block diagonal when the training size $N \to \infty$.
\end{theorem}
%This lemma highlights that, if the input-space distribution is approximately invariant to the transformation $G$, then assuming independence induces little error compared to the full covariance posterior.

Theorem~\ref{lem:decoupled} indicates that the independent variational distributions $\prod_{t=0}^{T-1} q_t(\bu_t)$ in HVGPs accurately approximate $q^{\star}(\bu_{0:T-1})$
%approximate the true GP posterior accurately 
when the input distribution %has symmetry
is symmetric under the transformation $G$. The symmetry further makes it easy for $\{G^t(\bZ)\}_{t=0}^{T-1}$ to match the input distribution, which by Theorem~\ref{lem:q-equi} renders that the variational approximation in Eq.~\eqref{eq:full-q-h} with the optimal $q^{\star}(\bu_{0:T-1})$ would be close to the true posterior. 
%Therefore, to make the HVGP approximation accurate, we should choose the transformation $G$ under which the input distribution $p$ is invariant. %has symmetry. \sjx{same meaning as the first sentence?}

%Theorem~\ref{lem:q-equi} and 
%Theorem~\ref{lem:decoupled} indicates that the independent variational distributions in HVGPs produce accurate approximations 
%%approximate the true GP posterior accurately 
%when the data distribution has symmetry under the transformation $G$.
%\emph{the transformation orbit of $G$ spread the input space well}. 
We illustrate the gist with a \textit{flip-mnist} problem, where each digit in the MNIST dataset is randomly flipped up-and-down or left-and-right. 
For a RBF kernel, we consider two variations of HVGPs in terms of the transformation: 1) Negation. We split input dimensions into two groups and negate them separately, resulting in a $4\times 100$ model. 2) Flipping the image up-and-down or left-and-right, resulting in a $4\times 100 $ model. We compare them with SVGPs using $100, 400$ inducing points, shown in Figure~\ref{fig:flipmnist}.
This experiment highlights the importance of matching the transformation $G$ with the data distribution.

\subsection{Computational Cost}

Assume that we have a $J$-way transformation, and each way is $\tilde{T}_j$-cyclic.
After decomposition this results in $\tilde{T} = \prod_{j=1}^J \tilde{T}_j$ complex-valued kernels and subsequently, $T = \prod_{j=1}^J (\lfloor  \tilde{T}_j / 2 \rfloor + 1) \geq \tilde{T} / 2^J$ real-valued kernels.
Let $T\times m$ represent using $m$ inducing points for each $t \in \{0, \dots, T-1\}$, and assume the mini-batch size $B = \bigO(m)$. 

\paragraph{Time Complexity.} 
The computational cost boils down to the cost of 
computing $k_t(\bZ_t, \bZ_t)$ and the cost of variational inference. 
To compute $k_t(\bZ_t, \bZ_t)$, we need the kernel values $k(\bZ_t, G^s(\bZ_t))$ for $s=0,...,\tilde{T}-1$.
%where $T^o = \prod_{j=1}^J T_j^o$ is the $J$-way period before resolving Hermitian kernels and $T = \prod_{j=1}^J (\lfloor  T_j^o / 2 \rfloor + 1) \geq T^o / 2^J $. 
If we assume the cost of applying $G$ is $c_G$, then computing $\bK_{\bu, \bu}$ requires $\bigO(Tm \times \tilde{T} c_G + Tm^2 \times \tilde{T})$ operations. 
Variational inference costs $\bigO(T m^3)$ time. 
Therefore, the overall complexity is
%\begin{align}
$\bigO(T m^3  + 2^J T^2 m^2 + 2^J T^2 m c_G)$.
%\end{align}
We note $2^J \leq T$, and for a single-way transformation, the cost simplifies to $\bigO(T m^3  + T^2 m^2 + T^2 m c_G)$.
In contrast, a SVGP with $Tm$ inducing points has the time complexity $\bigO(T^3m^3)$. 
% If $G$ involves only period $2$ transformations, then $T=T^o$ and the $2^J$ can be omitted in the cost.
Furthermore, HVGPs support straightforward parallelisms
%: we can
by locating computations of $k_t$ %on the $t$-th
on separate devices. %and parallelize them.

\paragraph{Space Complexity.} For computing $k_t(\bZ_t, \bZ_t)$, we need the kernel values $k(\bZ_t, G^s(\bZ_t))$ for $s=0,\dots,\tilde{T}-1$, which implies the memory cost $\bigO(T m^2 \times \tilde{T} )$.
Adding the $\bigO(T m^2)$ memory for keeping variational approximations, the overall space complexity is, $\bigO(2^J T^2m^2)$.

%% file: kern-tex/related-works.tex
\section{Related Works}

%\paragraph{Inter-Domain Inducing Points} \citep{lazaro2009inter, van2020framework, rudnerinter} generalizes the idea of \textit{inducing points} by summarizing the Gaussian Process by linear projections insteads of pointwise evaluations. In consequence, inter-domain inducing points could capture global structures using inducing functions. Specifically, they derive explicit expressions for particular inducing functions with squared exponential kernels. But the integration involved are intractable in general. On the other hand, Variational Fourier features (VFF) \citep{hensman2017variational} directly specifies the kernel function $k_{\bbf \bu}$ instead of the inducing function, then the kernel matrix $\bK_{\bu \bu}$ is obtained by the RKHS inner products of kernel functions. However, VFF still need to deal with the intractable integration in general \citep{dutordoir2020sparse, burt2020variational}.

%\textbf{Fourier Analysis of Kernels.} 
%Bochner's Theorem \citep{bochner1959lectures} establishes a bijection between complex-valued stationary kernels and positive finite measures using Fourier transform, thus providing an approach to formulate stationary kernels in the spectral domain 
The idea of applying Fourier analysis to kernels goes back at least to \citet{bochner1959lectures}.
In machine learning, this led to a flowering of large-scale kernel methods based on random features~\citep{rahimi2008random,yu2016orthogonal,dao2017gaussian}.
Bochner's theorem also allows designing stationary kernels by modeling a spectral density~\citep{wilson2013gaussian,samo2015generalized,parra2017spectral,benton2019function}.
On hyperspheres, zonal kernels are the counterpart of stationary kernels.  %$k(\bx, \bx') = \tau(\bx^{\top}\bx')$ 
%$\Sphere^{d-1}$
Their spectral decomposition is given by spherical harmonics~\citep{thomson1888treatise, morimoto1998analytic}.
Although closely related, none of these works have considered the discrete Fourier transform adopted in our method.
%There are a number of works that use Bochner's theorem to design stationary kernels~\citep{wilson2013gaussian, samo2015generalized, remes2017non, benton2019function}. It also flourishes random Fourier features \citep{rahimi2008random, avron2016quasi, yu2016orthogonal}, which uses random samples from the spectral measure for approximate feature maps. 

%provides the spectral decomposition for, 
%as a counterpart of the stationary kernels on $\Reals^d$. 
%Our transformation orbit $\{G^0, G^1, \dots, G^{T - 1}\}$ forms a cyclic group of order $T$ with $G$ being the generator. When including multi-way transformations, the orbit algebraically forms a finite Abelian group. For a larger range of groups, \citet{fukumizu2008characteristic} studies how the Fourier transform of functions in the RKHS indicaties whether a probability distribution can uniquely characterized by its kernel mean embedding. 

%\textbf{Structured Inducing Points.} 
HVGPs share many similarities with the works that propose decoupled~\citep{cheng2017variational, salimbeni2018orthogonally} and orthogonal~\citep{shi2020sparse} inducing points.
In particular, \citet{shi2020sparse} is also based on an orthogonal decomposition of the kernel and uses distinct groups of inducing points for them.
However, their decomposition %is based on Nystr{\"o}m approximation and therefore 
involves matrix inversion while ours can be computed using fast Fourier transforms.
 
%In particular,
%However, their approach is 
%Decoupled inducing points \citep{cheng2017variational, salimbeni2018orthogonally, shi2020sparse} explore structures in GP representations and use two sets of inducing points to scale up the model. 

Structured Kernel Interpolation (SKI) \citep{wilson2015kernel,wilson2015thoughts, evans2018scalable, izmailov2018scalable} places inducing points on a grid, leading to a structured $\bK_{\bu\bu}$ that allows fast matrix-vector multiplications.
For one-dimensional data, SKI exploits the Toeplitz structure of $\bK_{\bu\bu}$ generated by stationary kernels.
They first embed the Toeplitz matrix into a circulant matrix $\bC$, and use the fact that circulant matrices can be diagonalized by the DFT~\citep{tee2007eigenvectors} to enable fast computations:
\begin{align}
	\bC = \bF^{-1} \diag(\bF \bc) \bF.
\end{align}
Here $\bF$ is the DFT matrix, and $\bc$ is the first column of $\bC$.
%leading to a structued $\bK_{\bu \bu}$. % with fast matrix operations. % In consequence, SKI speeds up the GP computations and scales up to millions of inducing points. 
%In particular, a Toeplitz structured $\bK_{\bu \bu}$ corresponds to the kernel being invariant to translations, and enables fast matrix operations by padding it as a circulant matrix and using fast Fourier transforms. 
%The invariance and the Fourier transform highlights a similarity with HVGP, but major differences exists between them. 
%Compared to the whole grid in SKI, HVGP forms an orbit separately for every inducing point. In consequence, HVGP can not only optimize the locations of inducing points, but also support structured global transformations such as image flipping in Figure~\ref{fig:flipmnist}.
%\textbf{Circulant Matrix Diagonalizations.} 
%Given a vector $\bc \in \Complex^T$, a circulant matrix $\bC \in \Complex^{T \times T}$ is a square matrix from shifting $\bc$ to the right $t$ elements in $t$-th row. $\bC$ can be diagonalized via the discrete Fourier transform \citep{tee2007eigenvectors},
%\begin{align}
%	\bC = \bF^{-1} \diag(\bF \bc) \bF,
%\end{align}
%where $\bF$ is the DFT matrix of order $T$. We observe $\bF$ are unnormalized eigenvectors of $\bC$ and $\bF \bc$ are eigenvalues up to a factor. 
This equation highlights a connection with our HKD:
If we let $\bC = [k(G^{t_1}(\bx'), G^{t_2}(\bx))]_{t_1, t_2=0}^{T-1}$, then $\bc = [k(\bx, G^0(\bx')), ..., k(\bx, G^{T-1}(\bx'))]$ is the sequence we constructed in Eq.~\eqref{eq:dft-seq}, 
%Let 
%$\bc = [k(\bx, G^0(\bx')), ..., k(\bx, G^{T-1}(\bx'))]$, is the kernel matrix, and 
and the eigenvalues $\bF \bc$ %are exactly the kernels 
recover our decomposition $[k_t(\bx, \bx')]_{t=0}^{T-1}$ by the discrete Fourier transform.
In other words, our approach generalizes the structure of one-dimensional equally-spaced grids in SKI into arbitrary cyclic groups.
Moreover, our method allows trainable inducing locations, which plays an important role in combating the curse of dimensionality.

Besides inducing points in the data space, a number of works have investigated inducing features in the frequency domain.
%Variational Fourier features \citep{lazaro2009inter,hensman2017variational,dutordoir2020sparse,burt2020variational} investigate kernel spectral properties to develop variational inducing features, while the resulting features are usually not trainable and limited at low dimensional problems. 
However, these inducing features are either limited to specific kernels \citep{lazaro2009inter,hensman2017variational} or involve numerical approximations \citep{dutordoir2020sparse,burt2020variational}. 
The implementation of \citet{dutordoir2020sparse} only supports data up to 8 dimensions.

%Specifically, VISH \citep{dutordoir2020sparse} relies on the spherical harmonics (SHs) on the sphere $S^d$, but their implementation supports only up to 8 dimensions. 

Incorporating invariances with respect to input-space transformations into Gaussian processes is also investigated in a stream of works \citep{ginsbourger2016degeneracy, van2019learning}. 
Our work is orthogonal to them since we are not designing invariant models. 
Instead, we proposed a general inference method for GPs that can benefit from invariances in the data distribution.
%However our work is orthogonal to them since we focus on the invariance of the kernel with respect to the transformation regardless of the function. \sjx{we are if we interpret it as applying DFT to f \sim GP(0, k).}
%These works are also related to 
Relatedly, \citet{solin2020hilbert, borovitskiy2020} studied Gaussian processes on Riemannian manifolds. 

%% file: kern-tex/experiments.tex
\begin{figure*}[t]
    \centering
    % \hspace{-0.4cm}
    \includegraphics[width=\textwidth]{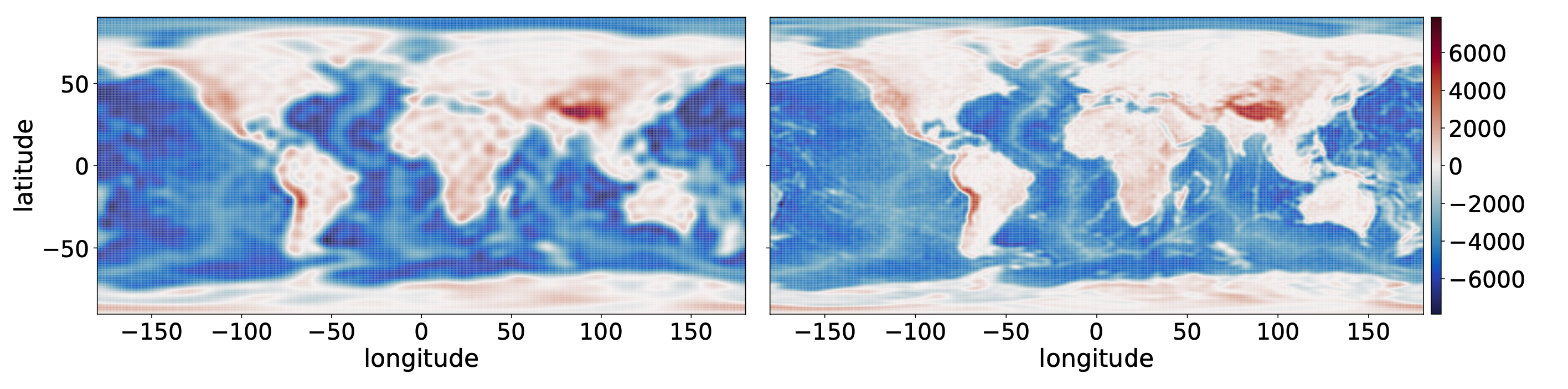}
    \caption{Predictive means of Earth elevations. We compare \textit{left}: SVGP (M=1000) and \textit{right}: HVGP ($13 \times 1000$). The transformation in the HVGP moves each point eastwards by $15^{\text{o}}$ longitude. We observe that the HVGP ($13 \times 1000$) fits the data more finely. \label{fig:elevation}}
    \vspace{-0.2cm}
\end{figure*}

\section{Experiments}
We present empirical evaluations in this section. All results were obtained using NVIDIA Tesla P100 GPUs, except in Sec~\ref{subsec:reg-ben} we used NVIDIA Tesla T4.
%for which our cluster supports parallelisms. 
Code is available at \url{https://github.com/ssydasheng/Harmonic-Kernel-Decomposition}.

\subsection{Earth Elevation}

We adopt GPs to fit the ETOPO1 elevation data of the earth \citep{amante2009etopo1}. ETOPO1 bedrock models the Earth's elevations from the bedrock surface underneath the ice sheets. A location is represented by the (longitude, latitude) pair, where longitude $\in [-180, 180]$ and latitude $\in [-90, 90]$. We build the dataset by choosing a location every $0.1$ degrees of longitude and latitude, resulting in $6,480,000$ data points. The dataset is randomly split for $72\%$ training,  $8\%$ validating, and $20\%$ testing. We use a three dimensional RBF kernel between the Euclidean coordinates of any two (longitude, latitude) locations. Because moving two locations eastwards by the same amount of longitudes preserves their Euclidean distance, the kernel is invariant to the $T$-cyclic longitude translation: 
\begin{align}
G([\mathrm{lon}, \mathrm{lat}]^{\top}) = [\mathrm{lon} + \frac{360}{T}, \mathrm{lat}]^{\top},
\end{align}
We set the period $T=12$ and $T=24$, so that $G$ moves a point eastwards by $30$ and $15$ degrees, respectively. Then we resolve Hermitian kernels to obtain $\lfloor 12 / 2 \rfloor + 1 = 7, \lfloor 24 / 2 \rfloor + 1 = 13$ real-valued kernels following Sec~\ref{subsec:resolve-hermitian}. 

%\begin{table}[h]
%\centering
%\begin{tabular}{cccccc}
%\toprule
% Model        & Test RMSE   & Test NLLD    & Train Time & Test Time \\
%\midrule 
%SKI      & 0.195 & 1.389 & 0.16 h & 1400s \\
%1k       & 0.252 & 0.040 & 0.38h & \\
%3k       & 0.208 & -0.146 & 2.47h & 55s \\
%5k       & - & - & 8.70h & 137s \\
%7x1k     & 0.189 & -0.246 & 1.55h & 61s \\
%13x1k    & 0.177 & -0.314 & 4.25h & 168s \\
%\bottomrule
%\end{tabular}
%\caption{Performances for Earth elevation. \label{tab:elevation}}
%\end{table}
%
%\begin{table}[t]
%\centering
%\begin{tabular}{cccccc}
%\toprule
% Model        & Test RMSE   & Test NLL    & Train T. & Test T.  \\
%\midrule 
%%SKI      & 0.195 & 1.389 & 0.16h  \\
%SKI      & \textbf{0.145} & 1.313 & 1.18h & 306s \\
%1k       & 0.252 & 0.040 & 0.38h  & 5s \\
%3k       & 0.208 & -0.146 & 2.47h & 22s\\
%5k       & 0.196 & -0.203 & 8.70h & 54s\\
%7x1k     & 0.189 & -0.246 & 1.55h & 32s\\
%13x1k    & \textbf{0.177} & \textbf{-0.314} & 4.25h & 85s  \\
%\bottomrule
%\end{tabular}
%\caption{Test performances on Earth elevation. We also include both training and testing time. \label{tab:elevation}}
%\end{table}

\begin{table}[t]
\centering
\small
\begin{tabular}{ccccc}
\toprule
 Model        & Test RMSE   & Test NLL    & Time \\
\midrule 
%SKI      & 0.195 & 1.389 & 0.16h  \\
SKI      & \textbf{0.145} & 1.313 & 1.18h \\
\midrule
1k       & 0.252 & 0.040 & 0.38h  \\
3k       & 0.208 & -0.146 & 2.47h \\
5k       & 0.196 & -0.203 & 8.70h \\
7x1k     & 0.189 & -0.246 & 1.55h \\
13x1k    & \textbf{0.177} & \textbf{-0.314} & 4.25h  \\
\bottomrule
\end{tabular}
\caption{Test performances on Earth elevations. \label{tab:elevation}}
\end{table}

We compare SVGPs with $1k, 3k, 5k$ inducing points and the HVGPs with $7 \times 1k, 13 \times 1k$ inducing points. We parallelize HVGPs using 4 GPUs, while SVGPs use only 1 GPU since it cannot be easily parallelized. All models are optimized using the Adam optimizer with learning rate 0.01 for 100K iterations. We also compare with SKI \citep{wilson2015kernel}. SKI runs into an out-of-memory error because of the large dataset, so we train it using a random $600,000$ subset of the training data. The performances are shown in Table~\ref{fig:elevation} and the predictive means are visualized in Figure~\ref{fig:elevation}. From both the table and the figure, we observe using more inducing points in variational GP models fits the dataset substantially better. Moreover, because of the decomposed structures and the parallelisms, HVGPs use more inducing points but run faster. In comparison, SKI uses $1M$ inducing points and achieves the best RMSE, but its NLL is much worse compared to variational GPs.

% Compared to SVGP with $3000, 3500$ inducing points, HVGP $7 \times 500$ achieves comparable performances and runs faster. % We also note the speed of HVGP can be vastly increased further by multi-device parallelisms.

\begin{figure}[t]
    \centering
    \includegraphics[width=\columnwidth]{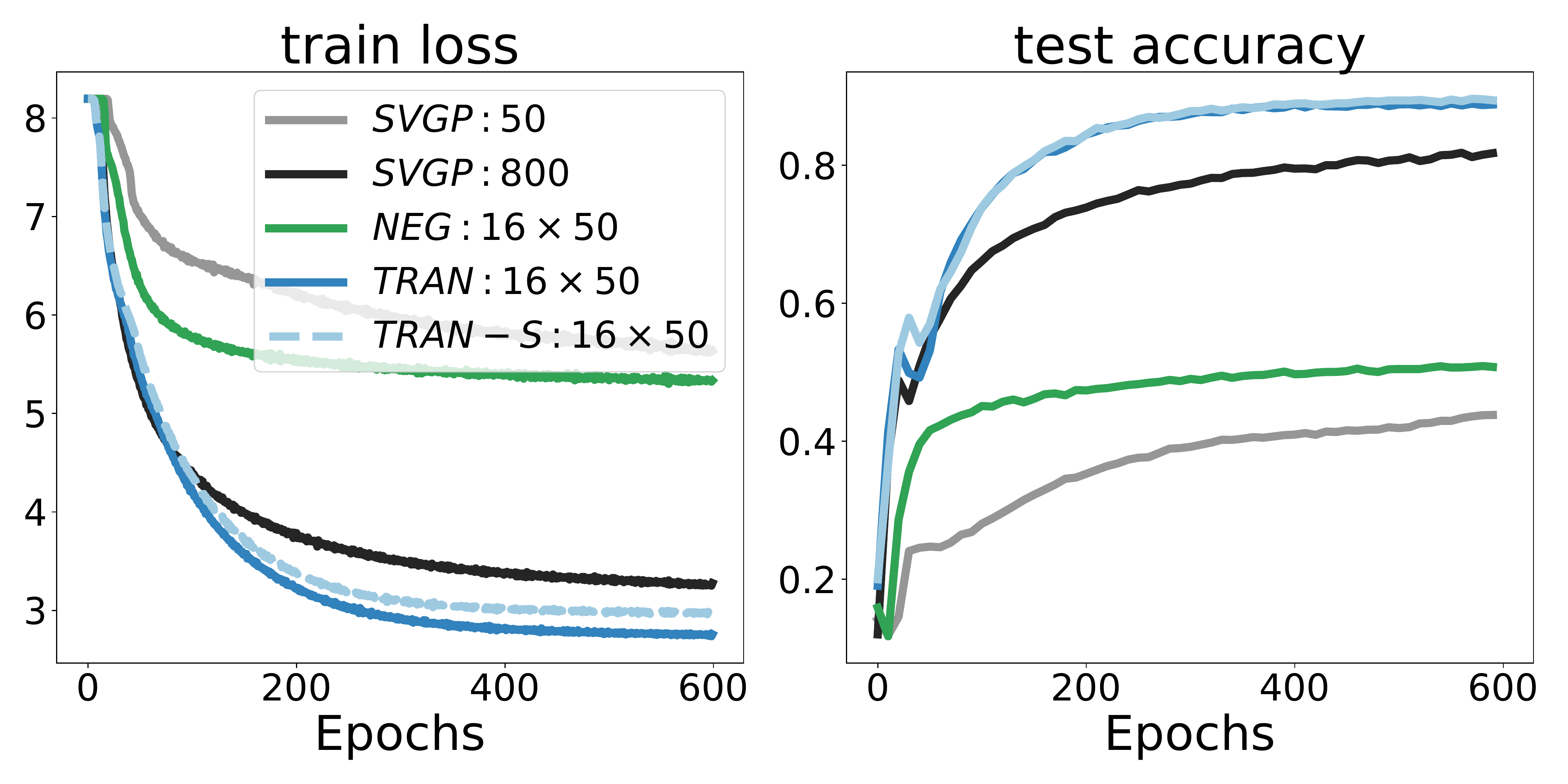}
    \caption{Translate-MNIST.  We plot how \textit{left: train loss} and \textit{right: test acc} evolve with training. We compare $16 \times M$ HVGPs with 1) \textit{NEG}: negations; 2) \textit{TRAN}: translations; 3) \textit{TRAN-S}: translations with shared inducing points. We observe the HVGPs using translations even outperform the SVGP with \textit{$800$} inducing points while the HVGP using negations performs substantially worse. Moreover, though the HVGP with \textit{TRAN-S} has only $50$ trainable inducing points, it performs similarly with the \textit{TRAN} model. \label{fig:trans-mnist}}
    \vspace{-0.2cm}
\end{figure}

\begin{figure*}[t]
\centering
\includegraphics[width=0.98\textwidth]{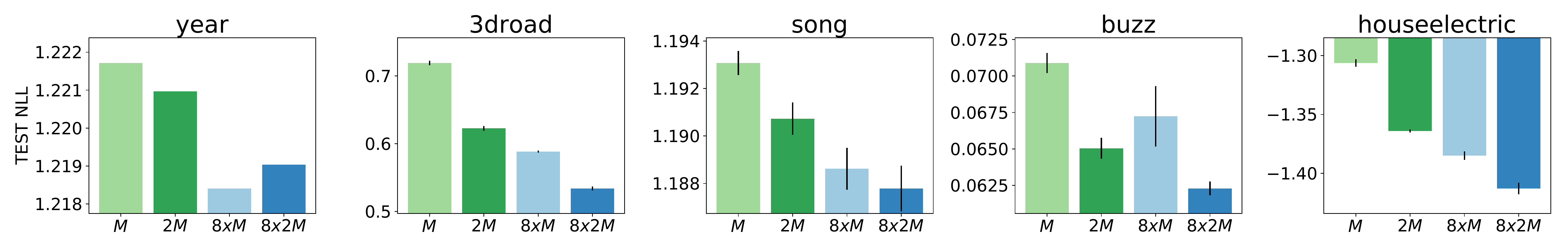}
\caption{Test negative log-likelihoods on regression benchmarks. We compare using $M, 2M, 8\times M, 8\times 2M$ for $M=1000$. We observe that the $8\times M, 8 \times 2M$ outperform the standard $M$ and $2M$ inducing points for the most datasets. \label{fig:reg}}
\end{figure*}

\begin{figure*}[t]
\centering
\includegraphics[width=0.98\textwidth]{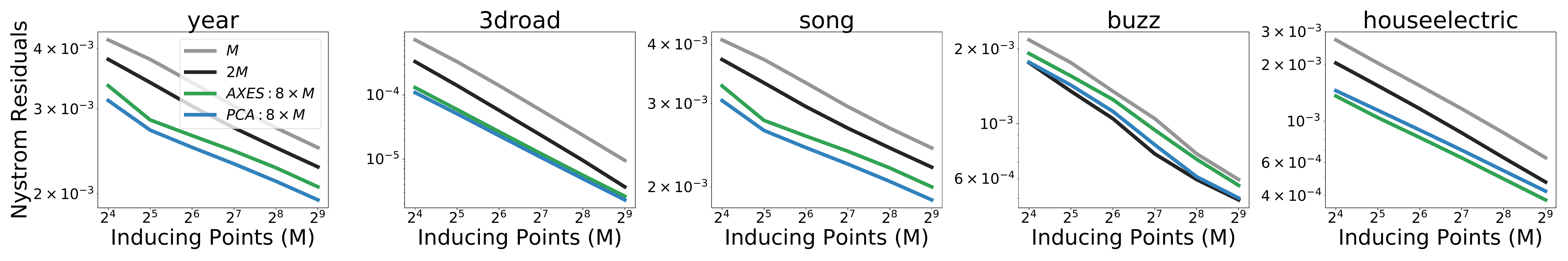}
\caption{%Accuracy of the Nystr\"om approximations in terms of 
Nystr\"om approximation errors measured by $\tr(\bK_{\bbf\bbf} - \bK_{\bbf \bu} \bK_{\bu \bu}^{-1} \bK_{\bu \bbf})$. 
We compare SVGPs using $M, 2M$ inducing points, and HVGPs with $8 \times M$ inducing points. For the transformation in HVGPs, we include both the negation along the standard axes and along the principal components. %From the figures, 
We observe that negations along PCA directions usually outperform negations along standard axes, by demonstrating smaller approximation errors. And both variations of HVGPs perform better than the SVGPs. We also observe a consistency between the Nystr\"om approximation errors and the regression performances. For example, the SVGP ($2M$) performs better than the HVGP ($8\times M$) for the \textit{buzz} dataset in Figure~\ref{fig:reg}, and this is similarly reflected by the Nystr\"om approximation errors.
\label{fig:nystrom-appr}}
\end{figure*}

\subsection{Translate-MNIST}
%\RBG{Be consistent about whether to hyphenate translate-MNIST.} 
The Elevation experiment uses one-way translations in HVGPs. In this section, we consider two-way translations for the \textit{Translate-MNIST} dataset. The dataset is obtained by translating every MNIST image leftwards and downwards by random numbers of pixels. We use an RBF kernel with shared lengthscales. The HVGP uses a 2-way transformation $G$ by translating the image leftwards or downwards by $4$ pixels. Since the MNIST images are of size $28\times 28$, $G$ is $(7, 7)$ cyclic. After resolving Hermitian kernels, we arrive at $(1 + \lfloor 7 /2 \rfloor) \times (1 + \lfloor 7 /2 \rfloor) = 16$ groups.

We compare the $16 \times 50$ translation HVGP with SVGPs using $50$ and $800$ inducing points. We further consider a variation of the HVGP by sharing the inducing points $\bZ$ as in Theorem~\ref{lem:q-equi}. We also include a $16 \times 50$ HVGP with 4-way negations whose transformation does not match the input-space distribution. We optimize all models using the Adam optimizer with learning rate 0.001 for 100K iterations. The results are shown in Figure~\ref{fig:trans-mnist}.

\subsection{Regression Benchmarks}\label{subsec:reg-ben}

We also evaluate our method on standard regression benchmarks, whose training data sizes range from $200$ thousand to $1$ million. Following \citet{wang2019exact}, we use the Mat\'ern 3/2 kernel with shared lengthscales. We consider the $J$-way composition of negations. %For example, if $\bx=[\bx_1, \bx_2]^{\top} \in \Reals^2$, the two-dimensional transformations are $G^{(0,0)}(\bx)=[\bx_1, \bx_2]^{\top}, G^{(0,1)}(\bx)=[\bx_1, -\bx_2]^{\top}, G^{(1,0)}(\bx)=[-\bx_1, \bx_2]^{\top}, G^{(1,1)}(\bx)=[-\bx_1, -\bx_2]^{\top}$. 
Specifically, we conduct negations over PCA directions. We split the PCA directions into $J$ subsets, and applying negations over these subsets results in a $2^J \times m$ model.  A visual comparison between the negation along axes and the negation along principal directions is shown in Figure~\ref{fig:pca}.

%Specifically, we can split the input dimension into $J$ subsets and conduct negations on each subset, so that we decompose the kernel into $T = 2^J$ kernels. However, if the data distribution is not isotropic, the negated transformation might be out-of-distribution. To make the $\{G^t(\bZ)\}_{t=0}^{T-1}$ spread over the input space, we choose to conduct negations along the principle component  analysis (PCA) directions. This comparison is visualized in Figure~\ref{fig:pca}.

% Moreover, if the input distribution is a zero-mean Gaussian,  not only negating along principle direction makes $G(\bZ)$ in-distribution, the input distribution is also invariant to $G$. This property is visualized in Figure~\ref{fig:pca}.

\begin{figure}[h]
\centering
\includegraphics[width=\columnwidth]{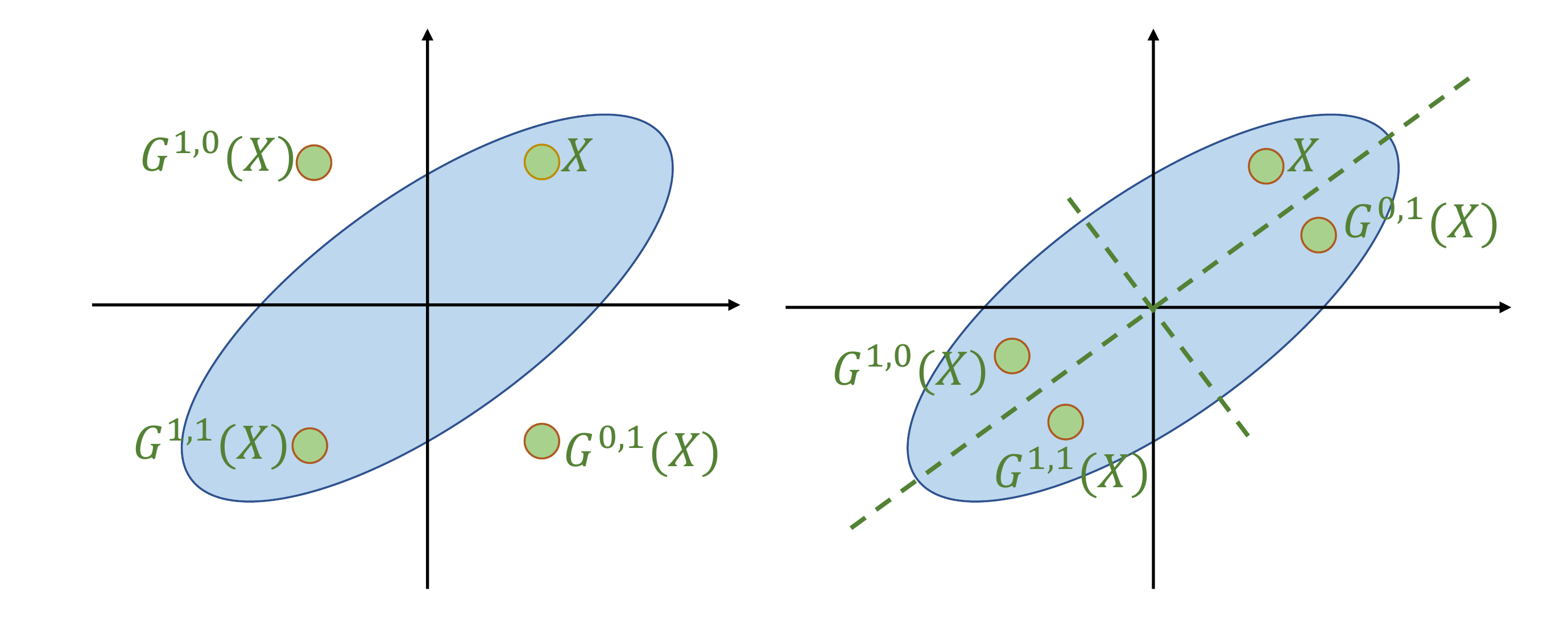}
\caption{Negation along axes (\textit{left}) and along principled directions (\textit{right}). The shaded area represents the input distribution. We observe that $G^{1,0}(\bx), G^{0,1}(\bx)$ are out of the data distribution when transforming along axes. In comparison, when transforming along PCA directions, the whole orbit is in-distribution.\label{fig:pca}}
\end{figure}

We compare SVGPs using $M$ and $2M$ inducing points with HVGPs using $8\times M$ and $8 \times 2M$ inducing points for $M=1000$. For HVGPs, we use $3$-way negations over PCA directions, and we use $8$ GPUs to place the computations of each GP in parallel.  The results for negative log likelihoods (NLLs) are reported in Figure~\ref{fig:reg}. We also report the root mean squared error (RMSE) performances in Figure~\ref{app:fig:reg-rmse} in the Appendix.

\begin{figure}[t]
\centering
\includegraphics[width=\columnwidth]{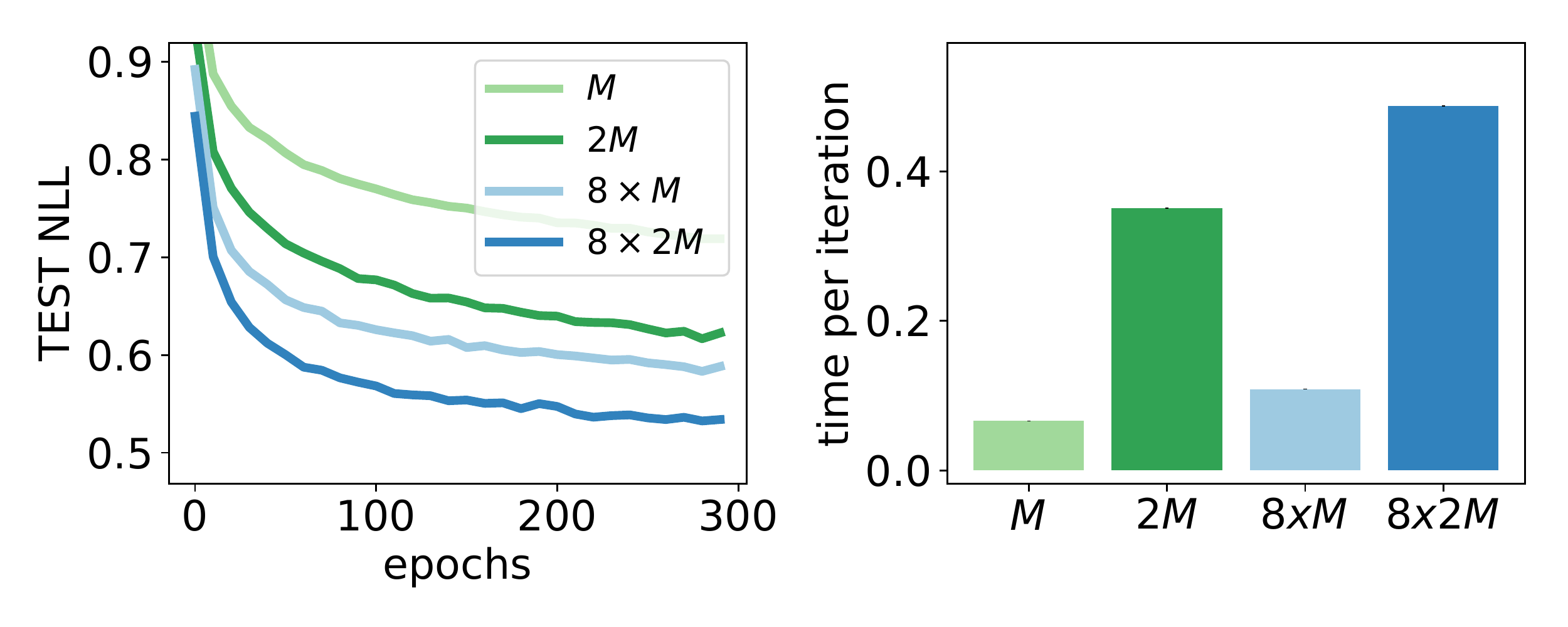}
\caption{Test negative log-likelihoods during training and the training time per iteration for the \textit{3droad} dataset.\label{fig:reg-training}}
\end{figure}

In Figure~\ref{fig:reg-training} we plot the evolution of the test negative log-likelihoods during training, and the training time per iteration, for the \textit{3droad} dataset. In Figure~\ref{fig:reg-training}, we observe that using more inducing points enables learning the dataset more quickly and converging to a better minima. Furthermore, due to the benefit of parallelism, the $8\times M$ HVGP has comparable running time compared to the standard SVGP with $M$ inducing points. And it is much faster compared to the SVGP with $2M$ inducing points in spite of the improved performance. Moreover, the computational bottleneck of SVGPs lies in the Cholesky decomposition, which does not support easy parallelism.

The performance of variational GPs relies largely on how well can the inducing points summarize the dataset, which can be measured by the accuracy of Nystr\"om approximation $\bK_{\bbf \bbf} \approx \bK_{\bbf \bu} \bK_{\bu \bu}^{-1} \bK_{\bu \bbf}$ \citep{drineas2005nystrom, titsias2009variational, burt2019rates}. 
%In this section 
We compare the Nystr\"om approximation errors with all methods, using the trace norm defined as $\tr(\bK_{\bbf \bbf} - \bK_{\bbf \bu} \bK_{\bu \bu}^{-1} \bK_{\bu \bbf})$. 
%Specifically 
For HVGPs, 
%the Nystr\"om error 
the trace norm is computed as
\begin{align}
    \sum_{t=0}^{T-1}\tr(\bK_{t, \bbf \bbf} - \bK_{t, \bbf \bu} \bK_{t, \bu \bu}^{-1} \bK_{t, \bu \bbf}).
\end{align}
where we use $\bK_{t, \cdot}$ to represent the kernel $k_t$. For each dataset, we randomly sample $3000$ points as $\bX$ and use a Mat\'ern 3/2 kernel whose lengthscales are set based on the median %bandwidth 
heuristic. We initialize the inducing points $\bZ$ using K-means and optimize them to minimizing the trace error. We compare SVGPs, HVGPs with negations along axes, and HVGPs with negations along principled directions. The results are shown in Figure~\ref{fig:nystrom-appr}.

\subsection{CIFAR-10 Classification}
In this subsection we  conduct experiments on the CIFAR10 classification problem using deep convolutional Gaussian processes \citep{blomqvist2019deep, dutordoir2019translation}, which combine the deep GP with the convolutional inducing features \citep{van2017convolutional}. Following the settings in \citet{shi2020sparse}, we compare HVGP with SVGP on both one-layer and multi-layer convolutional GPs.

% In this subsection we conduct experiments on the CIFAR10 classification problem using deep convolutional GPs \citep{blomqvist2019deep, dutordoir2019translation}. Following the settings in \citet{shi2020sparse}, we compare HVGPs with SVGPs on both one-layer and multi-layer convolutional GPs. 
For HVGPs we use the negation transformations on the inducing filters $G(\bz)=-\bz$. We compare HVGPs using \textcolor[HTML]{3182bd}{2xM}, \textcolor[HTML]{3182bd}{4xM} inducing points with SVGPs using \textcolor[HTML]{31a354}{M}, \textcolor[HTML]{31a354}{2M} inducing points. For the HVGP (\textcolor[HTML]{3182bd}{4xM}), we use $4$ GPUs to achieve parallelism. We also compare with the 2-way decomposed model in \citet{shi2020sparse} termed as \textcolor[HTML]{d95f0e}{M+M}. The results are summarized in Table~\ref{tab:cifar}. We observe that using more inducing filters results in better performances. In particular, the HVGP (\textcolor[HTML]{3182bd}{4xM}) achieves the best NLLs. Because of the parallelism, the HVGP (\textcolor[HTML]{3182bd}{4xM}) also has comparable running time with the HVGP (\textcolor[HTML]{3182bd}{2xM}), and both are faster than \textcolor[HTML]{31a354}{2M} and \textcolor[HTML]{d95f0e}{M+M} for deep models.

\begin{table}[t] \vskip -0.05in
\resizebox{\columnwidth}{!}{%
\begin{tabular}{ccccc}
\toprule
M                          & Model        & ACC   & NLL    & sec/iter \\
%\midrule
%%VGG16 \citep{zhang2017noisy}   & -  & 81.79 & -   & -  \\
%$[1]$                          & -               & 77.04 & -0.73   & -  \\
%$[2]$                          & 384, 384, 384   & 76.17 & -0.69   & -  \\
%$[3]$                          & 384, 384, 1000  & 78.76 & -0.88   & 0.42  \\
%$[3]$                          & 768, 768, 2000  & 80.33 & -0.82   & 1.25  \\
\midrule  
%\multicolumn{2}{c}{\citep{shi2020sparse}}      & 80.33*  & -0.82* & 1.25* \\
%\midrule
\multirow{5}{*}{384x0, 1K}      & \textcolor[HTML]{31a354}{M}   & 65.70$\pm$0.06 &  1.65$\pm$0.00 & 0.21  \\
                                & \textcolor[HTML]{31a354}{2M}   & \textbf{67.84$\pm$0.07} &  1.52$\pm$0.00 & 0.39  \\
                                & \textcolor[HTML]{d95f0e}{M+M}  & 67.67$\pm$ 0.07 & \textbf{1.50 $\pm$0.01} & 0.39 \\
                                & \textcolor[HTML]{3182bd}{2xM} & 66.26$\pm$1.11 &  1.76$\pm$0.17 & 0.45  \\
                                & \textcolor[HTML]{3182bd}{4xM}  & 67.76$\pm$0.05 &  \textbf{1.51$\pm$0.01} & 0.52  \\
\midrule 
\multirow{5}{*}{384x1, 1K}      & \textcolor[HTML]{31a354}{M}   & 76.40$\pm$0.02 &  1.03$\pm$0.00 & 0.16  \\
                                &\textcolor[HTML]{31a354}{2M}   & 77.11$\pm$0.10 &  1.00$\pm$0.00 & 0.47  \\
                                & \textcolor[HTML]{d95f0e}{M+M}  & \textbf{77.48$\pm$0.10} & 0.98$\pm$0.01 & 0.41 \\
                                & \textcolor[HTML]{3182bd}{2xM}  & 77.09$\pm$0.18 &  1.00$\pm$0.00 & 0.37  \\
                                & \textcolor[HTML]{3182bd}{4xM}  & 77.30$\pm$0.17 &  \textbf{0.95$\pm$0.00} & 0.36  \\
\midrule 
\multirow{5}{*}{384x2, 1K}      & \textcolor[HTML]{31a354}{M}    & 79.01$\pm$0.11 &  0.86$\pm$0.00 & 0.17  \\
                                &\textcolor[HTML]{31a354}{2M}   & 80.27$\pm$0.04 &  0.81$\pm$0.00 & 0.52  \\
                                & \textcolor[HTML]{d95f0e}{M+M}  & 79.98 $\pm$0.21 & 0.80$\pm$0.01  & 0.46 \\
                                & \textcolor[HTML]{3182bd}{2xM}   & 80.04$\pm$0.04 &  0.80$\pm$0.00 & 0.37  \\
                                & \textcolor[HTML]{3182bd}{4xM}   & \textbf{80.52$\pm$0.20} &  \textbf{0.75$\pm$0.01} & 0.37  \\
\midrule 
\multirow{5}{*}{384x3, 1K}      & \textcolor[HTML]{31a354}{M}  & 82.41$\pm$0.08 &  0.73$\pm$0.01 & 0.40  \\
                                &\textcolor[HTML]{31a354}{2M}   & - & - & -  \\
                                & \textcolor[HTML]{d95f0e}{M+M}  & 83.26$\pm$0.19 & 0.69$\pm$0.01 & 1.24 \\
                                & \textcolor[HTML]{3182bd}{2xM}   & \textbf{84.97$\pm$0.08} &  0.60$\pm$0.00 & 0.90 \\
                                &  \textcolor[HTML]{3182bd}{4xM}  & \textbf{84.85$\pm$0.11} &  \textbf{0.58$\pm$0.00} & 0.90  \\
\bottomrule
\end{tabular}}
\caption{Deep Convolutional GPs for CIFAR-10 classification. Previous SOTA \citep{shi2020sparse} achieves ACC=$80.33$, NLL=0.82, and $1.25$ sec/iter.
We use $384\text{x}\ell, 1K$ to represent a $(\ell+1)$-layer model with a respective number of inducing points in each layer. 
We compare \textcolor[HTML]{31a354}{M}, \textcolor[HTML]{31a354}{2M}, \textcolor[HTML]{d95f0e}{M+M}, \textcolor[HTML]{3182bd}{2xM},  \textcolor[HTML]{3182bd}{4xM}. We used $4$ GPUs for the \textcolor[HTML]{3182bd}{4xM} model to achieve parallelism. For the four-layer model, using \textcolor[HTML]{31a354}{2M} inducing points did not fit in memory. Instead we used a model with~$(700\text{x}3, 1600)$ inducing points and achieved ACC=$82.89\pm 0.05$, NLL=$0.73 \pm 0.00$, and $1.10$ sec/iter. \label{tab:cifar}}
\end{table} 

%% file: kern-tex/appendix.tex
\section{Kernel Fourier Transform}
We have shown that the Fourier series of length $T$ form the harmonic kernel decomposition with $T$ kernels. Intuitively, if $T \to \infty$, we obtain a “continuous” frequency representation of the kernel, which would be akin to a Fourier transform. % In this subsection, we present rigorous evidences to connect the kernel Fourier series with kernel Fourier transform. % In particular, we generalize the Fourier transformation from stationary kernels and zonal kernels to any $G$-invariant kernels of some periodic transformation $G$.

Consider the transformation $G^\bs: \xdomain \to \xdomain, \bs \in \Reals^J$, corresponding to $J$-way transformations. We assume the transformation $G$ is $\bone$-periodic: $G^\bzero(\bx)=G^\bone(\bx)=\bx, G^{\bs_1+\bs_2}(\bx)=G^{\bs_1}(G^{\bs_2}(\bx)), \forall \bs_1, \bs_2 \in \Reals^J$. A kernel is $G$-invariant if for any $\bs \in \Reals^J$, $k(G^\bs(\bx), G^\bs(\bx')) = k(\bx, \bx')$. 

Given the inputs $\bx, \bx'$, we consider the space of kernel values: $k(\bx, G^{\bs}(\bx')), \bs \in \Reals^J$. For $\bt \in \Reals^J$, we define the complex-valued function $k_{\bt}: \xdomain \times \xdomain \to \Complex$ using the Fourier transform,
\begin{align}
    k_{\bt}(\bx, \bx') = \int_{\Reals^J} e^{-2\pi i \bs^{\top}\bt} k(\bx, G^\bs(\bx')) \mathrm{d}\bs,
\end{align}
In this way, $k_\bt(\bx, \bx')$ captures the frequency of $\bt$  in the function $\bs \to k(\bx, G^\bs(\bx'))$. Similar to the harmonic kernel decomposition, we show an alternative representation of the kernel using $k_t$.
\begin{theorem}[Harmonic Kernel Representation]\label{thm:kern-ft}
\begin{align}
    k(\bx, \bx') = \int_{\Reals^J} k_\bt(\bx, \bx') \mathrm{d} \bt.
\end{align}
Moreover, $k_\bt$ is a kernel for all $\bt \in \Reals^J$.
\end{theorem}

\begin{proof}[\textbf{Proof of Theorem~\ref{thm:kern-ft}}]
We prove this theorem by the following derivation,
\begin{align}
    \int_{\Reals^J} k_{\bt}(\bx, \bx') \mathrm{d}\bt     &=  \int_{\Reals^J} \int_{\Reals^J} e^{-2\pi i \bs^{\top}\bt} k(\bx, G^\bs(\bx')) \mathrm{d}\bs \mathrm{d}\bt 
    = \int_{\Reals^J}  k(\bx, G^\bs(\bx'))  \int_{\Reals^J}e^{-2\pi i \bs^{\top}\bt} \mathrm{d}\bt \mathrm{d}\bs \notag \\
    &= \int_{\Reals^J}  k(\bx, G^\bs(\bx'))  \delta_{\bs} \mathrm{d}\bs 
    = k(\bx, \bx'). \notag
\end{align}
where we used the property that the Fourier transform of the constant function is the delta function.

To show that $k_t$ is a kernel, we prove the following equality,
\begin{align}
    \int_{\Reals^J}  \int_{\Reals^J} e^{-2\pi i \bt^{\top}(\bs_2 - \bs_1)} k(G^{\bs_1}(\bx), G^{\bs_2}(\bx')) \mathrm{d}\bs_1 \mathrm{d}\bs_2
    &=\int_{\Reals^J}  \int_{\Reals^J} e^{-2\pi i \bt^{\top}\bs_2} k(\bx, G^{\bs_2}(\bx')) \mathrm{d}\bs_1 \mathrm{d}\bs_2 \notag \\
    &= \int_{\Reals^J} e^{-2\pi i \bt^{\top}\bs_2} k(\bx, G^{\bs_2}(\bx')) \mathrm{d}\bs_2  
    = k_\bt(\bx, \bx'). \notag
\end{align}
\end{proof}

We demonstrate the Kernel Fourier Transform by considering a stationary kernel on the unit circle. We denote the input $x$ as the angle, then the kernel admits the form $k(x, x') = \kappa(x - x')$, where $\kappa$ is a periodic function of period $2\pi$. Let $\kappa_0(t) = \kappa(t) \mathbb{I}[0 \leq t < 2\pi]$, then 
\begin{align}
    k(x, x') = \kappa(x -x') = \sum_{n \in \Ints} \kappa_0(x - x' - 2 \pi n), \notag 
\end{align}
Let $G^{s}(x) = x + 2\pi s$, we obtain, 
\begin{align}
    k_t(x, x') 
    &= \sum_{n \in \Ints}  \int_{\Reals} e^{-2\pi i s t} \kappa_0(x - x' - 2\pi s - 2 \pi n) \mathrm{d}s 
    = \frac{1}{2\pi}\sum_{n \in \Ints}  \int_{\Reals} e^{-i (x-x'-2\pi n - w) t} \kappa_0(w) \mathrm{d}w \notag \\
    &= e^{-it(x - x')}  \left[\frac{1}{2\pi} \int_{\Reals} e^{i w t} \kappa_0(w)  \mathrm{d}w \right] \sum_{n \in \Ints} e^{2\pi i nt } 
    = e^{-it(x - x')} \hat{\kappa}_0(t) \sum_{n \in \Ints} \delta(t - n). \notag 
\end{align}
where $\hat{\kappa}_0$ is the inverse Fourier transform of $\kappa_0$. Then we have the Fourier series,
\begin{align}
    k(x, x') = \int_\Reals k_t(x, x') dt = \sum_{n \in \Ints} \hat{\kappa}_0(n) e^{-i n(x - x')}. \notag 
\end{align}

\section{Inter-domain Inducing Points Formulation}

We present an inter-domain inducing points interpretation of the harmonic kernel decomposition. An inter-domain inducing point is a function $w: \xdomain \to \Complex$ whose inducing variable is defined as,
\begin{align}
    u_w = \int f(\bx) w(\bx) d\bx,
\end{align}
We introduce $T$ kinds of inter-domain inducing points. For $t=0, ..., T-1$, given $\bz_t \in \xdomain$, the inter-domain inducing point of the $t$-th kind is,
\begin{align}
    w_t &= \sum_{s=0}^{T-1} \bF^H_{t, s} \delta_{G^s(\bz_t)}, \\
    u_{w_t} &= \int f(\bx) w_t(\bx) d\bx = \sum_{s=0}^{T-1} \bF^H_{t, s} f(G^s(\bz_t)),
\end{align}
Therefore, we can generalize the kernel function to include inter-domain inputs,
\begin{align}
    k(\bx, w_t) &= \expect[f(\bx) u^H_{w_t}] = \sum_{s=0}^{T-1} \bF_{t, s} k(\bx, G^s(\bz_t)), \\
    k(w_t, w'_t) &= \expect[u_{w_t}u^H_{w'_t} ] = \sum_{s=0}^{T-1}\sum_{s'=0}^{T-1} \bF^H_{t, s} \bF_{t, s'} k(G^s(\bz_t), G^{s'}(\bz'_t))= \sum_{s=0}^{T-1} \bF_{t, s} k(\bz_t, G^s(\bz'_t)), 
\end{align}
where the last equality is based on Lemma~\ref{app:lem:kfs-1}. Furthermore, for $0 \leq t \neq t' \neq T-1$,
\begin{align}
    k(w_{t}, w_{t'}) &= \sum_{s=0}^{T-1}\sum_{s'=0}^{T-1} \bF_{t, s} \bF_{t', s'}^H k(G^s(\bz_t), G^{s'}(\bz'_{t'})) = 0. \notag 
\end{align}
where the last equality is based on Lemma~\ref{app:lem:kfs-orth}. We observe,
\begin{align}
    k(\bx, w_t) = k_t(\bx, \bz_t); \; k(w_t, w'_t) = k_t(\bz_t, \bz'_t).
\end{align}
Now we find that \emph{the proposed inter-domain inducing points formulation is equivalent to the kernel Fourier series.} Furthermore, the equivalence also reinterprets HVGPs as standard SVGPs using inter-domain inducing points while enforcing block diagonal posterior covariances.

\section{More Experiments and Details}

%
%\begin{figure*}[t]
%\centering
%\includegraphics[width=0.98\textwidth]{figures/dataset_nys_residual_matern32_NX3000.pdf}
%\caption{Accuracy of the Nystr\"om approximations in terms of $\tr(\bK_{\bbf\bbf} - \bK_{\bbf \bu} \bK_{\bu \bu}^{-1} \bK_{\bu \bbf})$, where the inducing points are randomly sampled from the dataset excluding the training set. We compare SVGPs using $M, 2M$ inducing points, and HVGPs with $T \times M$ inducing points. We use $T=4$ (light color) and $T=16$ (dark color). For the transformation in HVGPs, we include both the negation along the standard axes and along the principle directions. From the figures, we observe that negations along PCA directions usually outperform negations along standard axes, by demonstrating smaller approximation errors. And both variations of HVGPs perform better than the SVGPs. We also observe a consistency bewtween the Nystr\"om approximation errors with the regression performances. Specifically, for the regression experiments, the SVGP ($2M$) only performs on par with the HVGP ($8\times M$) for the \textit{buzz} dataset, and this is similarly reflected by the Nystr\"om approximation errors.
%\label{fig:nystrom-appr}}
%\end{figure*}

\subsection{Toy Visualization}

HVGPs are based on the decomposed GP formulation and assume independent variational posteriors. Therefore, the predictions on a target location $\bx^{\star}$ can be decomposed as the combination of independent elements,
\begin{align}
    &\normal(\bzero, \bK_{\star \star} - \sum_{t=0}^{T-1} \bK_{t, \star \bu} \bK_{t, \bu \bu}^{-1} \bK_{t, \bu \star})  ) + \sum_{t=0}^{T-1} \normal(\bK_{t, \star \bu_t} \bK_{t, \bu_t \bu_t}^{-1} \bmu_t, \bK_{t, \star \bu_t} \bK_{t, \bu_t \bu_t}^{-1}  \bS_t \bK_{t, \bu_t \bu_t}^{-1} \bK_{t, \star \bu_t} ) \notag
\end{align}
where we use $\bK_{t, \cdot}$ to represent the kernel $k_t$. The first term in the prediction represents the error of the Nystr\"om approximation, and the remaining terms contain the predictions from all subprocesses.

In this section we conduct a Snelson's 1D toy experiment to visualize the posterior predictions and each term. We set $T=2, G(x)=-x, m=5$, which results in HVGP ($2\times 5$). Because the original training inputs are positive, we preprocess it by subtracting the inputs by the mean. 
The results are shown in Figure~\ref{fig:toy}. We find that using $2 \times 5$ inducing points fit the training data well, and generate reasonable predictive uncertainty as well. The predictions for the two GPs correspond to the symmetric and the antisymmetric fraction, respectively.

\begin{figure}[t]
\centering
\includegraphics[width=0.9\textwidth]{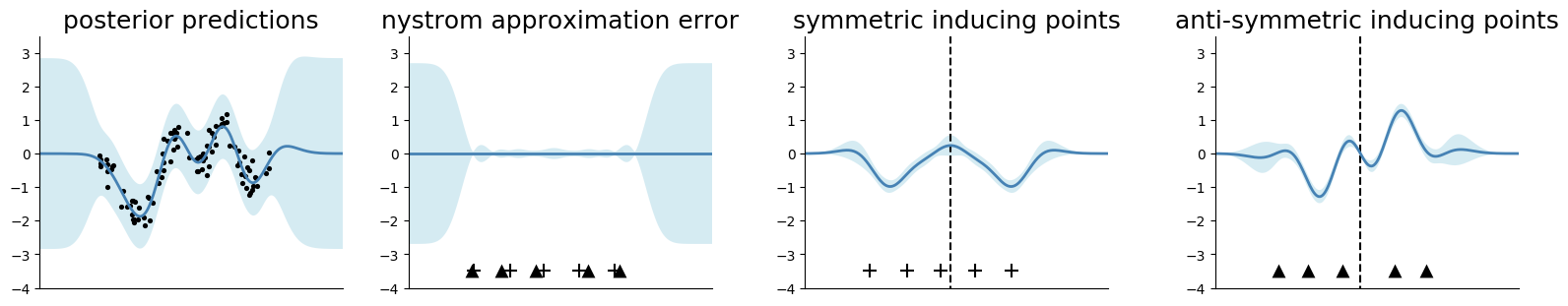}
\caption{Posterior predictions on the Snelson dataset, where shaded bands correspond to intervals of $\pm 3$  standard deviation. The leftmost \textit{posterior prediction} plot is the combination of the right three plots, plus the observation variance. We also visualize the associated inducing points for each plot. We observe that the HVGP predictions are separated as the symmetric fraction and the antisymmetric fraction.\label{fig:toy}}
\end{figure}

\subsection{Regression Benchmarks}

We use the Mat\'ern 3/2 kernel with shared lengthscales across input dimensions. 
For HVGPs, the transformation $G$ is by negating over PCA directions. We split the PCA directions into $J$ subsets, then apply negations over which results in a $2^J \times M$ model. We let the $j$-th subset contain the directions with the $j \text{th large}, (J+j) \text{th large}, ...$ eigenvalues, so that the principal subspace is covered well. Except for the \textit{year} dataset which has a standard train/test split, each dataset is randomly split into $64\%$ training, $16\%$ validating, $20\%$ testing sets and is averaged over $3$ random splits. We initialize the inducing points using K-means and initialize the kernel lengthscale using the median heuristic. The Gaussian likelihood variance is initialized at $0.1$. For all experiments, we optimize for 30k iterations with the Adam optimizer using learning rate 0.003 and batch size 256. We visualize the results for test RMSEs in Figure~\ref{app:fig:reg-rmse}, and how each criterion evolves along training in Figure~\ref{app:fig:reg-training-all}.

\begin{figure}[h]
\centering
\includegraphics[width=0.98\textwidth]{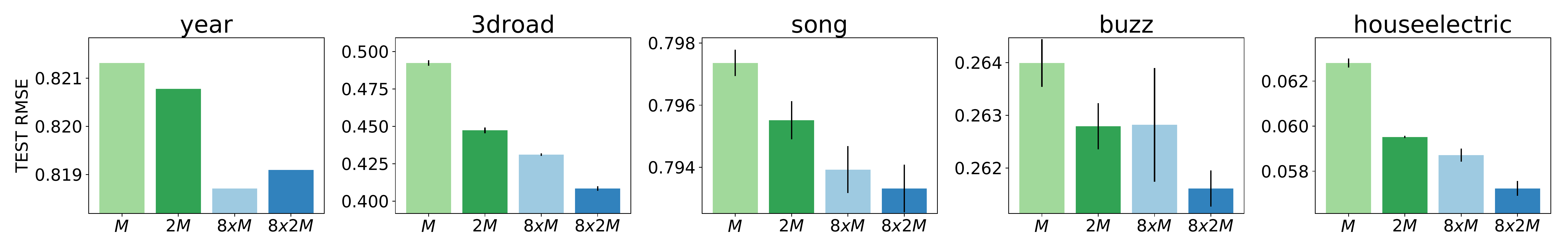}
\caption{Test RMSEs on regression benchmarks. We compare SVGPs using $M, 2M$ inducing points and HVGPs using $8\times M, 8\times 2M$ inducing points, for $M=1000$. \label{app:fig:reg-rmse}}
\end{figure}

\begin{figure*}[t]
\centering
\includegraphics[width=0.98\textwidth]{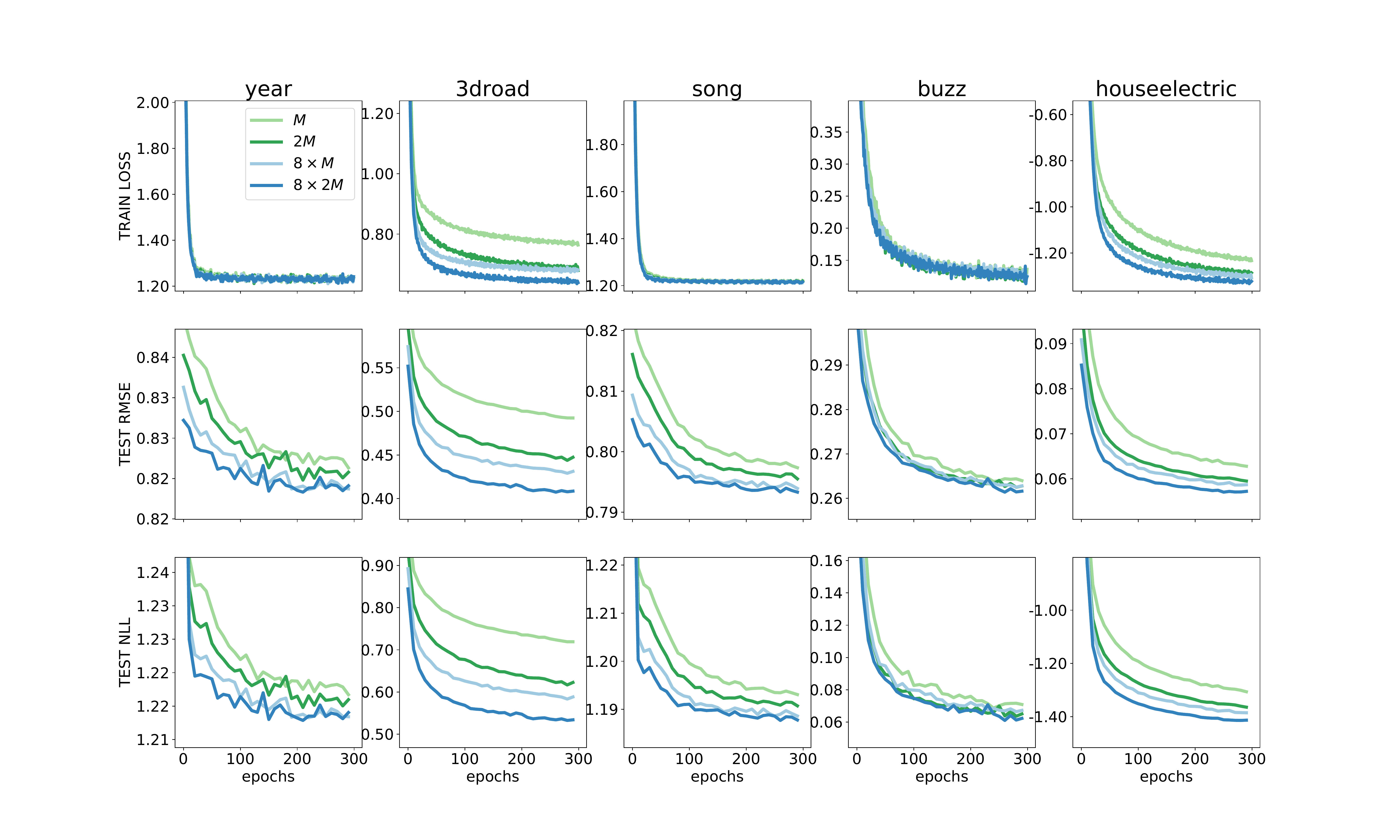}
\caption{How \textit{train loss}, \textit{test rmse}, and \textit{test nll} evolve during training. We compare SVGPs using $M, 2M$ inducing points and HVGPs using $8\times M, 8\times 2M$ inducing points, for $M=1000$.\label{app:fig:reg-training-all}}
\end{figure*}

%
%\subsection{Nystr\"om Approximation}
%
%
%The performance of SVGPs relies largely on whether the inducing points can summarize the dataset well, which is measured by the Nystr\"om approximation $\bK_{\bbf \bu} \bK_{\bu \bu}^{-1} \bK_{\bu \bbf} \approx \bK_{\bbf \bbf}$ \citep{titsias2009variational, burt2019rates}. In this section we directly compare Nystr\"om approximation of various approachs, using the trace norm $\tr(\bK_{\bbf \bbf} - \bK_{\bbf \bu} \bK_{\bu \bu}^{-1} \bK_{\bu \bbf})$. Specifically for HVGPs, the Nystr\"om error is computed as,
%\begin{align}
%    \sum_{t=0}^{T-1}\tr(\bK_{t, \bbf \bbf} - \bK_{t, \bbf \bu} \bK_{t, \bu \bu}^{-1} \bK_{t, \bu \bbf}).
%\end{align}
%Within each dataset, we randomly sample $3000$ points as $\bX$ and use a RBF kernel whose lengthscales are set based on the median bandwidth heuristic. Excluding $\bX$, We also randomly sample $m$ points as $\bZ$. For HVGPs, we share the inducing points $\bZ$ for $t=0, ..., T-1$. We compare SVGPs, HVGPs with negations along axes, and HVGPs with negations along principled directions. The results are shown in Figure~\ref{fig:nystrom-appr}.

\subsection{CIFAR-10 Classification}

%\begin{table*}[h]
%\centering
%\begin{tabular}{ccccc}
%\toprule
%             & 1-layer   & 2-layer   & 3-layer   & 4-layer   \\ 
%\midrule
%filter size  & 5         & 5, 4      & 5,4,5     & 5,4,5,4   \\ 
%stride size  & 1         & 1,2       & 1,2,1     & 1,1,1,1   \\ 
%channel num  & -         & 10        & 10,10     & 16,16,16  \\ 
%pooling      & -         & -         & -         & mean      \\ 
%pooling size & -         & -         & -         & 1,2,1     \\ 
%padding      & SAME      & SAME      & SAME      & SAME      \\ 
%M            & 384x0, 1k & 384x1, 1k & 384x2, 1k & 384x3, 1k \\ 
%\bottomrule
%\end{tabular}
%\caption{Model Configurations for Deep Convolutional Gaussian processes. \label{app:tab:dcgp-conf}}
%\end{table*}
%
%\begin{figure}[h]
%\centering
%\includegraphics[width=0.2\textwidth]{figures/convgp-G.pdf}\caption{Group Splitting for the $4\times M$ model. \label{app:fig-conv-g}}
%\end{figure}

\begin{minipage}{\textwidth}
  \hspace{2em}
  \begin{minipage}[b]{0.2\textwidth}
    \centering
\includegraphics[width=\textwidth]{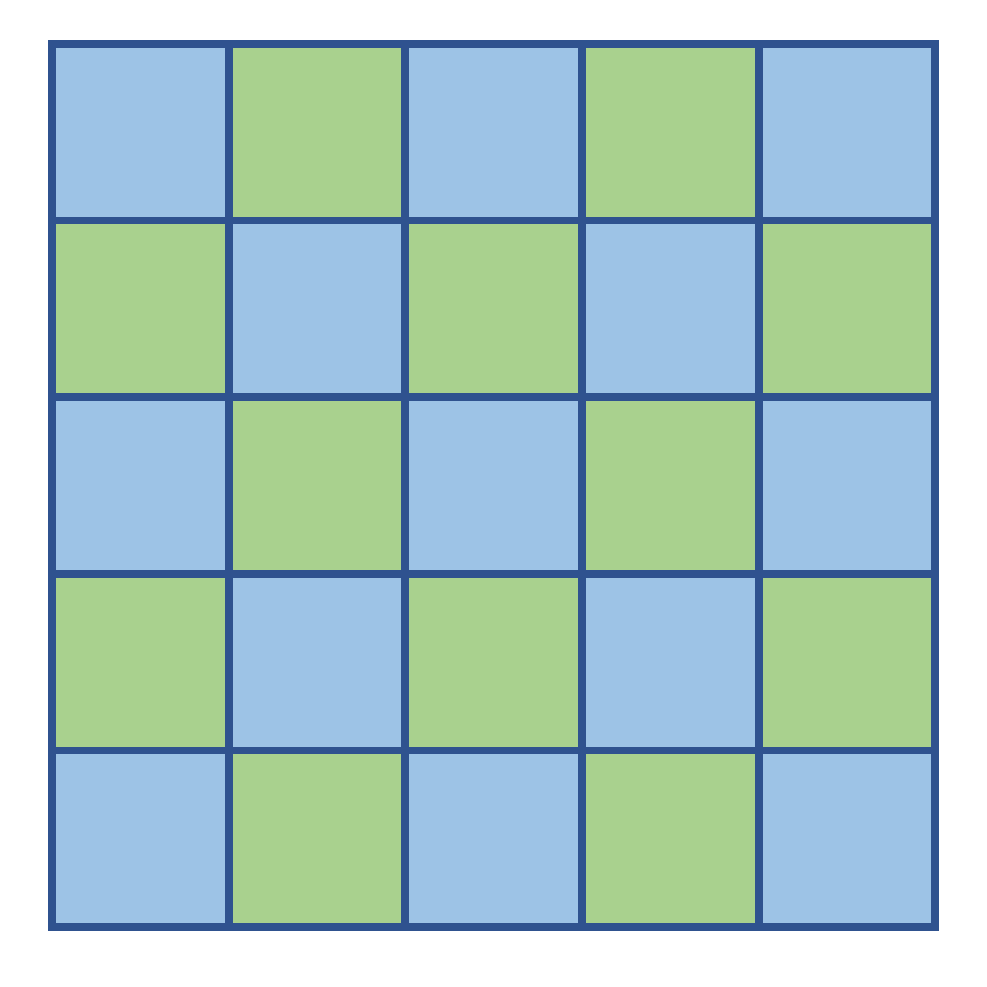}\captionof{figure}{Group Splitting for the HVGP ($4\times M$). \label{app:fig-conv-g}}
  \end{minipage}
  \hspace{-2em}
  \begin{minipage}[b]{0.8\textwidth}
    \centering
    \begin{tabular}{ccccc}
\toprule
             & 1-layer   & 2-layer   & 3-layer   & 4-layer   \\ 
\midrule
filter size  & 5         & 5, 4      & 5,4,5     & 5,4,5,4   \\ 
stride size  & 1         & 1,2       & 1,2,1     & 1,1,1,1   \\ 
channel num  & -         & 10        & 10,10     & 16,16,16  \\ 
pooling      & -         & -         & -         & mean      \\ 
pooling size & -         & -         & -         & 1,2,1     \\ 
padding      & SAME      & SAME      & SAME      & SAME      \\ 
M            & 384x0, 1K & 384x1, 1K & 384x2, 1K & 384x3, 1K \\ 
\bottomrule
\end{tabular}
      \captionof{table}{Model Configurations for Deep Convolutional Gaussian processes. \label{app:tab:dcgp-conf}}
    \end{minipage}
  \end{minipage}

For deep Gaussian processes, we let $M_l, h_{l}$ be the number of inducing points and the number of input units in the $l$-th layer, respectively. The variational posterior for the inducing points $\bU^l, \bU^l \in \Reals^{M_l \times h_{l+1}}$ in the $l$-th layer is usually a multivariate Gaussian,
\begin{align}
    q(\mathrm{vec}(\bU^l)) = \normal(\mathrm{vec}(\bM^l), \bSigma^l),
\end{align}
where $\bM^l \in \Reals^{M_l \times h_{l+1}}, \bSigma^l \in \Reals^{(M_l  h_{l+1})\times (M_l  h_{l+1})}$  are the mean and the covariance, respectively. A commonly-used structure for $\bSigma^l$ is the block-diagonal covariance \citep{salimbeni2017doubly} , i.e., assuming independence between output channels. However, the true posterior is not independent. Moreover, such covariance involves $h_{l+1}$  covariances of shape $M_l \times M_l$, which are both memory intensive and computation intensive. Therefore, following \citet{park2018deep}, we use the Kronecker-factored structure for the covariance, i.e., $\bSigma^l = \bSigma_o^{l} \otimes \bSigma_i^{l}$, where $\bSigma_o \in \Reals^{h_{l+1} \times h_{l+1}}, \bSigma_i \in \Reals^{M_l \times M_l}$ correspond to the output covariance and the input covariance, respectively.

Following \citet{shi2020sparse}, all models were optimized using 270k iterations with the Adam optimizer using a learning rate 0.003 and a batch size 64. We anneal the learning rate by 0.25 every 50k iterations to ensure convergence. Unlike \citet{shi2020sparse} which used a zero mean function, we used a convolution mean function whose filter is 1 for the center pixel and 0 everywhere else, since we observe it with a better performance. We used the robust multi-class classification likelihood. For lower layers in the deep convolutional GP, we used multi-output GPs for each input patch \citep{blomqvist2019deep}; for the output layer, we used the TICK kernel \citep{dutordoir2019translation}.  The patch kernels are RBF kernels with shared lengthscales, whose lengthscales and variances are initialized at $5$. The TICK location kernel is a Mat\'ern 3/2 kernel whose lengthscales and variances are initialized at $1$ and $3$, respectively. To initialize the inducing filters, we use K-means samples from $\min(100*M, 10000)$ random input patches, while the inputs in all layers are obtained by forwarding the image through a random Xavier convnet.

For the HVGP ($2\times M$) we use the negation transformation on the inducing points $G(\bz)=-\bz$. For the HVGP ($4 \times M$), we also use the negation transformation over two groups that are determined by pixel locations, as shown in Figure~\ref{app:fig-conv-g}.

\section{Proofs}

\subsection{Lemmas}

\begin{lemma}\label{app:lem:kfs-1} For any $t=0,...,T-1$, $\bx, \bx' \in \xdomain$,
\begin{align}
    \sum_{s=0}^{T-1}\sum_{s'=0}^{T-1} \bF^{H}_{t, s} \bF_{t, s'} k(G^s(\bx), G^{s'}(\bx'))
    = \sum_{s=0}^{T-1} \bF_{t, s} k(\bx, G^s(\bx')). \notag 
\end{align}
\end{lemma}
\begin{proof}[Proof] We prove the equality by the expression of $\bF$,
\begin{align}
    &\frac{1}{T^2}\sum_{s=0}^{T-1}\sum_{s'=0}^{T-1} e^{-i\frac{2\pi t}{T}(s'-s) }k(G^{s}(\bx), G^{s + (s'-s)}(\bx')) =\frac{1}{T}\sum_{s_0=0}^{T-1} e^{-i\frac{2\pi t }{T}s_0 }k(\bx, G^{s_0}(\bx')), \notag 
\end{align}
where we used the kernel invariance to $G$. Also, since $G$ is $T$-cyclic, we changed the variable $s'-s$ to $s_0$ and $s_0$ still ranges from $0$ to $T-1$.
\end{proof}

\begin{lemma}\label{app:lem:kfs-orth}For any $0 \leq t_1 \neq t_2 \leq T-1$, $\bx, \bx' \in \xdomain$,
\begin{align}
    \sum_{s_1=0}^{T-1} \sum_{s_2=0}^{T-1} \bF^H_{t_1,s_1} \bF_{t_2,s_2} k(G^{s_1}(\bx), G^{s_2}(\bx'))  = 0, \label{eq:kfs-orth-1} \\
     \sum_{s_1=0}^{T-1} \sum_{s_2=0}^{T-1} \bF_{t_1,s_1} \bF^H_{t_2,s_2} k(G^{s_1}(\bx), G^{s_2}(\bx'))  = 0, \label{eq:kfs-orth-2} 
\end{align}
\end{lemma}
\begin{proof}[Proof]
%We expand expressions for $\bF$,
Below we prove \eqref{eq:kfs-orth-1}. 
The proof of \eqref{eq:kfs-orth-2} follows similarly.
\begin{align}
& \sum_{s_1=0}^{T-1} \sum_{s_2=0}^{T-1} \bF^H_{t_1,s_1} \bF_{t_2,s_2} k(G^{s_1}(\bx), G^{s_2}(\bx')) = \sum_{s_1=0}^{T-1} \sum_{s_2=0}^{T-1}  \bF^H_{t_1,s_1} \bF_{t_2,s_1 + s_2} k(G^{s_1}(\bx), G^{s_1+s_2}(\bx')) \notag \\
&= \sum_{s_2=0}^{T-1} k( \bx, G^{s_2}(\bx')) \sum_{s_1=0}^{T-1}\bF^H_{t_1,s_1} \bF_{t_2,s_1 + s_2} 
= \sum_{s_2=0}^{T-1} k( \bx, G^{s_2}(\bx')) e^{-i \frac{2\pi t s_2}{T}}\sum_{s_1=0}^{T-1}\bF^H_{t_1,s_1} \bF_{t_2,s_1}
= 0. \notag
\end{align}
In the last step, $\sum_{s_1=0}^{T-1}\bF^H_{t_1,s_1} \bF_{t_2,s_1} = 0$ whenever $t_1 \neq t_2$. 
This is because the columns of $\bF$ form an orthogonal basis over the set of $T$-dimensional complex vectors.
%The other equation can be proved similarly.
\end{proof}

\begin{lemma}\label{app:lem:gp-decomp-value} Under the harmonic formulation, for $\bx \in \xdomain$,
\begin{align}
    f_t(\bx) = \sum_{s=0}^{T-1} \bF^H_{t,s} f(G^s(\bx)).
\end{align}
\end{lemma}
\begin{proof}[Proof] We consider a marginal distribution on a subset of function values,
\begin{align}
    p(\{f(G^s(\bx))\}_{s=0}^{T-1}, \{f_0(G^s(\bx))\}_{s=0}^{T-1}, ..., \{f_{T-1}(G^s(\bx))\}_{s=0}^{T-1}), \notag
\end{align}
The distribution can be represented as,
\begin{align}
    f(G^s(\bx)) &= \sum_{t=0}^{T-1} f_t(G^s(\bx)) , s=0,...,T-1,\notag \\
    \bbf_t(G^{0:T-1}(\bx)) &\sim \normal(0, \bK_t(G^{0:T-1}(\bx), G^{0:T-1}(\bx))), \notag 
\end{align}
We first investigate the structure of the kernel matrix $\bK_t$, 
\begin{align}
    k_t(G^j(\bx), G^{j'}(\bx)) 
    &= \sum_{s=0}^{T-1} \bF_{t,s} k(G^j(\bx), G^{s+j'}(\bx)) 
    =\sum_{s=0}^{T-1} \bF_{t,s} k(\bx, G^{s+j'-j }(\bx)) = \sum_{s=0}^{T-1} \bF_{t,s+j-j'} k(\bx, G^{s}(\bx)) \notag \\
    &= e^{-\frac{2\pi i t(j-j')}{T}}\sum_{s=0}^{T-1} \bF_{t,s} k(\bx, G^{s}(\bx)) = e^{-\frac{2\pi it(j-j')}{T}} k_t(\bx, \bx),\notag
\end{align}
Therefore, the matrix $\bK_t = [k_t(\bx, \bx) e^{-\frac{2\pi it(j-j')}{T}}]_{j,j'=0}^{T-1}$. Let $\epsilon_t \in \Reals$ be a random Gaussian noise, then the random vector of $f_t$ can be written as,
\begin{align}
    \bbf_t(G^{0:T-1}(\bx)) = [\sqrt{k_t(\bx, \bx)} e^{-\frac{2\pi it j}{T}}  \epsilon_t]_{j=0}^{T-1},
\end{align}

Now we can compute the RHS in the lemma,
\begin{align}
    \sum_{s=0}^{T-1} \bF^H_{t,s} f(G^s(\bx)) 
    &= \sum_{s=0}^{T-1} \bF^H_{t,s} \sum_{t'=0}^{T-1} f_{t'}(G^s(\bx)) = \sum_{t'=0}^{T-1}  \sum_{s=0}^{T-1} \bF^H_{t,s} f_{t'}(G^s(\bx)) , \notag
\end{align}
If $t'=t$, 
\begin{align}
    \sum_{s=0}^{T-1} \bF^H_{t,s} f_{t}(G^s(\bx)) 
    &=  \sum_{s=0}^{T-1} \bF^H_{t,s} \sqrt{k_t(\bx, \bx)} e^{-\frac{2\pi it s}{T}}  \epsilon_t
%    &= \sqrt{k_t(\bx, \bx)}\epsilon_t  \sum_{s=0}^{T-1} \bF_{t,s} e^{\frac{2\pi it s}{T}} \notag \\
    = \sqrt{k_t(\bx, \bx)}\epsilon_t = f_t(\bx).
\end{align}
If $t' \neq t$,
\begin{align}
    \sum_{s=0}^{T-1} \bF^H_{t,s} f_{t'}(G^s(\bx)) 
    &=  \sum_{s=0}^{T-1} \bF^H_{t,s} \sqrt{k_{t'}(\bx, \bx)} e^{-\frac{2\pi it' s}{T}}  \epsilon_t = \sqrt{k_{t'}(\bx, \bx)} \epsilon_t \sum_{s=0}^{T-1} e^{\frac{2\pi i(t-t')}{T}} = 0.
\end{align}
Therefore, 
\begin{align}
    \sum_{s=0}^{T-1} \bF^H_{t,s} f(G^s(\bx)) =f_t(\bx).
\end{align}
Because this holds for all marginal distributions, it holds as well for the function samples from Gaussian processes.\end{proof}

\subsection{Proofs for Sec~\ref{sec:kern-decomp}}\label{app:proof:ortho}
\begin{proof}[\textbf{Proof of Proposition~\ref{prop:invariance}}]
The equality can be directly proven,
\begin{align}
    k_t(\bx, G(\bx')) 
    &= \sum_{s=0}^{T-1} \bF_{t, s} k(\bx, G^{s+1}(\bx')) = \sum_{s=0}^{T-1} \bF_{t, s-1} k(\bx, G^{s}(\bx')) \notag \\
    &= \sum_{s=0}^{T-1} \frac{1}{T} e^{-i\frac{2\pi t (s-1)}{T} } k(\bx, G^{s}(\bx')) = e^{i\frac{2\pi t}{T}} k_t(\bx, \bx'). 
\end{align}
\end{proof}

\begin{proof}[\textbf{Proof of Theorem~\ref{thm:kern-decomp}}]\label{proof:app:kern-decomp}
The equality can be directly proven,
\begin{align}
   \sum_{t=0}^{T-1} k_t(\bx, \bx')
    &= \sum_{t=0}^{T-1} \sum_{s=0}^{T-1} \frac{1}{T}e^{-i\frac{2\pi st}{T}} k(\bx, G^s(\bx')) 
    =\frac{1}{T}  \sum_{s=0}^{T-1} k(\bx, G^s(\bx')) \sum_{t=0}^{T-1} e^{-i\frac{2\pi st}{T} } 
    = \frac{1}{T}  \sum_{s=0}^{T-1} k(\bx, G^s(\bx')) T\delta_{s} = k(\bx, \bx'). \notag
\end{align}
To prove that $k_t$ is a kernel, we observe from Lemma~\ref{app:lem:kfs-1} that $k_t(\bx, \bx') = \bF_{t,:}^H \bK \bF_{t,:}$, where $\bK = [k(G^{s_1}(\bx), G^{s_2}(\bx'))]_{s_1, s_2=0}^{T-1}$. Since $k$ is a kernel,  we conclude that $k_t(\bx, \bx') = \bF_{t,:}^H \bK \bF_{t,:}$ is a kernel as well.
%=======
%To prove that $k_t$ is a kernel, we claim that $k_t(\bx, \bx') = \bF_{:,t}^H \bK \bF_{:,t}$, where $\bK = [k(G^{s_1}(\bx), G^{s_2}(\bx'))]_{s_1, s_2=0}^{T-1}$. Expanding the matrix vector products,
%\begin{align}
%    &\frac{1}{T^2}\sum_{s_1=0}^{T-1}\sum_{s_2=0}^{T-1} e^{-i\frac{2\pi t}{T}(s_2-s_1) }k(G^{s_1}(\bx), G^{s_1 + (s_2-s_1)}(\bx')) \notag \\
%    &=\frac{1}{T}\sum_{s_0=0}^{T-1} e^{-i\frac{2\pi t }{T}s_0 }k(\bx, G^{s_0}(\bx')) = k_t(\bx, \bx'), \notag 
%\end{align}
%where we changed the variable $s_2-s_1$ to $s_0$. Furthermore, since $k$ is a kernel and $\bK$ is positive semidefinite,  we conclude that $k_t(\bx, \bx') = \bF_{:,t}^H \bK \bF_{:,t}$ is a kernel as well.
%\end{proof}
%
%>>>>>>> origin/master

%Expanding the matrix vector products,
%\begin{align}
%    &\frac{1}{T^2}\sum_{s_1=0}^{T-1}\sum_{s_2=0}^{T-1} e^{-i\frac{2\pi t}{T}(s_2-s_1) }k(G^{s_1}(\bx), G^{s_1 + (s_2-s_1)}(\bx')) \notag \\
%    &=\frac{1}{T}\sum_{s_0=0}^{T-1} e^{-i\frac{2\pi t }{T}s_0 }k(\bx, G^{s_0}(\bx')) = k_t(\bx, \bx'), \notag 
%\end{align}
%where we changed the variable $s_2-s_1$ to $s_0$. Furthermore, since $k$ is a kernel and $\bK$ is positive semidefinite,  we conclude that $k_t(\bx, \bx') = \bF_{t,:}^H \bK \bF_{t,:}$ is a kernel as well.
\end{proof}

\begin{proof}[\textbf{Proof of Lemma~\ref{lem:ortho}}]\label{proof:kern-ortho}
We firstly prove that, for all $t_1 \neq t_2$ and $\bx, \bx'$, $\innerprod{k_{t_1}(\cdot, \bx)}{k_{t_2}(\cdot, \bx')}_{\rkhs_k} = 0$. The RKHS inner product can be computed as,
\begin{align}
&\innerprod{k_{t_1}(\cdot, \bx)}{k_{t_2}(\cdot, \bx')}_{\rkhs_k} = \sum_{s_1=0}^{T-1} \sum_{s_2=0}^{T-1} \bF_{t_1,s_1} \bF^H_{t_2,s_2} k(G^{s_1}(\bx), G^{s_2}(\bx')) =0,
\end{align}
where the last equality is due to Lemma~\ref{app:lem:kfs-orth}.

Moreover, if the functions $f, g$ can be written as linear combinations of the corresponding kernels, 
\begin{align}
 f(\bx) = \sum_s a_s k_{t_1}(\bx, \bx^s_{t_1}) ; \;
 g(\bx) = \sum_s b_s k_{t_2}(\bx, \bx^s_{t_2}) \notag 
\end{align}
Following that $\innerprod{k_{t_1}(\cdot, \bx)}{k_{t_2}(\cdot, \bx')}_{\rkhs_k} $ for all $\bx, \bx'$, $\innerprod{f}{g}_{\rkhs_k} = 0$ as well.

Based on Moore-Aronszajn Theorem \citep{aronszajn1950theory, berlinet2011reproducing}, the RKHS spaces $\rkhs_{k_{t_1}}$ and $\rkhs_{k_{t_1}}$ are the set of functions which are pointwise limits of Cauchy sequences in the form $f_n(\bx) = \sum_s a_s k_{t_1}(\bx, \bx^s_{t_1}) $ and $g_n(\bx) = \sum_s b_s k_{t_2}(\bx, \bx^s_{t_2})$, respectively. Moreover, based on the \citet[Lemma 5]{berlinet2011reproducing}, the inner product of two pointwisely convergent Cauchy sequences also converges. We conclude that for any $f \in \rkhs_{k_{t_1}}, g \in \rkhs_{k_{t_2}}$, $\innerprod{f}{g}_{\rkhs_k} = 0$.
\end{proof}

\begin{proof}[\textbf{Proof of Proposition~\ref{prop:disjoint}}]
Without loss of generality, we only need to prove that,
\begin{align}
\rkhs_1 \cap \rkhs_2 = \{0\}, \notag
\end{align}
Firstly, $0 \in \rkhs_1, 0\in \rkhs_2$ because $\rkhs_1$ and $\rkhs_2$ are Hilbert spaces. Then we assume another function $f \neq 0$ and $f \in \rkhs_1 \cap \rkhs_2$. By Lemma~\ref{lem:ortho}, $\|f\|_{\rkhs_k} = \innerprod{f}{f}_{\rkhs_k} = 0$, which is contradictory to $f \neq 0$ and $\rkhs_k$ being a Hilbert space.
\end{proof}

\begin{proof}[\textbf{Proof of Theorem~\ref{thm:rkhs-decom}}]
We use $\rkhs_t$ to represent the RKHS corresponding to the kernel $k_t$. Given a function $f \in \rkhs_k$, we firstly assume $f$ can be written as a linear combination of the kernel functions,
\begin{align}
f(\bx) = \sum_s a_s k(\bx, \bx^s) , \notag 
\end{align}
Based on the kernel sum decomposition, we can rewrite $f$,
\begin{align}
    f(\bx) = \sum_s a_s \sum_{t=0}^{T-1} k_t(\bx, \bx^s) = \sum_{t=0}^{T-1} \underbrace{\sum_s a_s k_t(\bx, \bx^s) }_{:=f_t(\bx)}, \notag 
\end{align}
Because $f_t$ is a linear combination of $k_t$, $f_t \in \rkhs_{t}$, for $t=0,...,T-1$. Proposition~\ref{prop:disjoint} states that the RKHSs $\rkhs_{t_1}, \rkhs_{t_2}$ are disjoint except the zero function, thus $f = \sum_{t=0}^{T-1} f_t$ is a unique expansion of $f$ to these RKHSs. Moreover, we can represent the function $f$ alternatively,
\begin{align}
    f(\bx) &= \innerprod{f}{k(\bx, \cdot)}_{\rkhs_k} = \sum_{t=0}^{T-1}\innerprod{f}{k_t(\bx, \cdot)}_{\rkhs_k} 
    = \sum_{t=0}^{T-1}\innerprod{\sum_{t'=0}^{T-1} f_{t'}}{k_t(\bx, \cdot)}_{\rkhs_k}  = \sum_{t=0}^{T-1}\innerprod{f_t}{k_t(\bx, \cdot)}_{\rkhs_k}, \notag
\end{align}
where the last equality uses the orthogonality between RKHSs. By using the orthogonality again, we also show that,
\begin{align}
    \innerprod{f_t}{k_t(\bx, \cdot)}_{\rkhs_k} &= \innerprod{f_t}{\sum_{t'=0}^{T_1}k_{t'}(\bx, \cdot)}_{\rkhs_k} 
     = \innerprod{f_t}{k(\bx, \cdot)}_{\rkhs_k} = f_t(\bx),
\end{align}
Therefore, $f(\bx) = \sum_{t=0}^{T-1}\innerprod{f}{k_t(\bx, \cdot)}_{\rkhs_k}$ uniquely separates $f$ into these RKHSs. More generally, if $f$ is the pointwise limits of Cauchy sequences of functions in the form of linear combinations of the kernel function. Based on the \citet[Lemma 5]{berlinet2011reproducing}, the inner product of two pointwisely convergent Cauchy sequences also converges. We conclude that for any $f \in \rkhs_{k}$,
\begin{align}
    f(\bx) = \sum_{t=0}^{T-1}\innerprod{f}{k_t(\bx, \cdot)}_{\rkhs_k},
\end{align}
uniquely decomposes the function $f$ into the RKHSs $\rkhs_t, t=0,...,T-1$.

Based on \citet[Theorem 5]{berlinet2011reproducing}, the squared RKHS norm  of $f$ can be written as the sum of squared RKHS norms,
\begin{align}
    \|f\|_{\rkhs_k}^2 = \sum_{t=0}^{T-1} \|f_t\|_{\rkhs_t}^2 .\notag 
\end{align}

\end{proof}

\subsection{Proof of Sec~\ref{sec:hvgp}}\label{app:prof:lemma-dvi}

\begin{proof}[\textbf{Proof of Theorem~\ref{lem:q-equi}}]
%Let $\bK[\bx, \bx'] = [k(G^{t_1}(\bx), G^{t_2}(\bx'))]_{t_1, t_2=0}^{T_1}$. In Proof of Theorem~\ref{thm:kern-decomp} and Proof of Lemma~\ref{lem:ortho}, we have proven that,
%\begin{align}
%   \forall 0 \leq t \leq T-1, \; k_t(\bx, \bx') &= \bF_{t,:}^{H}\bK[\bx, \bx'] \bF_{t,:} \notag \\
%   \forall 0\leq t_1 \neq t_2 \leq T-1, \; 0 &= \bF_{t_1,:}^{H}\bK[\bx, \bx'] \bF_{t_2,:} , \notag
%\end{align}
%Therefore, for $\bZ \in \Reals^{m \times d}$,
%\begin{align}
%    \begin{bmatrix}
%    \bK_0(\bZ, \bZ) & & \\
%     & \ddots &  \\
%     & & \bK_{T-1}(\bZ, \bZ)
%    \end{bmatrix}
%\end{align}

Under the HVGP formulation, the inducing variable $u_t = f_t(\bz)$. From Lemma~\ref{app:lem:gp-decomp-value}, we have $f_t(\bz) = \sum_{s=0}^{T-1} \bF^H_{t, s} f(G^s(\bz))$.

Under the inter-domain formulation, let $w_t$ be the inter-domain inducing point corresponding to $\bz$ in $k_t$,
\begin{align}
     w_t = \sum_{s=0}^{T-1} \bF^H_{t, s} \delta_{G^s(\bz)}, \notag 
\end{align}
Then the inducing variable corresponding to $w_t$ is,
\begin{align}\label{app:eq:ut}
    u_{w_t} = \int f(\bx) w_t(\bx) d\bx = \sum_{s=0}^{T-1} \bF^H_{t, s} f(G^s(\bz)),
\end{align}
Therefore, the two inducing variables are the same,
\begin{align}
    u_{w_t} = u_t.
\end{align}

%This lemma uses the variational posterior, 
%\begin{align}
%    q(\vectorize{\bU}) = \normal(\vectorize{\bF \bM_v}, (\bI \otimes \bF) \bS_v (\bI \otimes \bF^H)), \notag 
%\end{align}
Therefore, the variational posterior under the harmonic formulation can be rewritten in an inter-domain SVGP form,
\begin{align}
    q^{inter}(f, \bU) = p(f|\bU; \{\bw_t\}_{t=0}^{T-1}) q(\vectorize{\bU}),
\end{align}
where $\bU :=[\bu_0, ..., \bu_{T-1}]^{\top}$ and $p$ is the inter-domain Gaussian process.

Now we connect the inter-domain SVGP to the standard SVGP using $\{G^t(\bZ)\}_{t=0}^{T-1}$. For the standard SVGP, the inducing variables are $\bv_t = f(G^t(\bZ))$, and $\bV:=[\bv_0, ..., \bv_{T-1}]^{\top} \in \Complex^{T \times m}$. For the inter-domain SVGP, the inducing variables are $\bu_t$. As shown in Eq.~\eqref{app:eq:ut}, $\bu_t = \sum_{s=0}^{T-1} \bF^H_{t, s} \bv_s$, then we have the equality,
\begin{align}
    \bU = \bF^H \bV,
\end{align}
Because of the bijective linearity, 
\begin{align}
p(f|\bU; \{\bw_t\}_{t=0}^{T-1}) = p(f|\bV; \{G^t(\bZ)\}_{t=0}^{T-1}),
\end{align}
Furthermore, the variational posterior for $\bV$ is $\normal(\vectorize{\bV}| \vectorize{\bM_{v}}, \bS_v)$, which is equivalent to the variational posterior for $\bU$,
\begin{align}
    q(\vectorize{\bU}) = \normal(\vectorize{\bU}| \vectorize{\bF^H \bM_{v}}, (\bI \otimes \bF^H) \bS_v  (\bI \otimes \bF)). \notag 
\end{align}
So the argument has been proved.

%
%Now we can compute the prediction distribution for $f$. Let $\bK_{f\bU} := k(\bx, [\bW_0^{\top}, ..., \bW_{T-1}^{\top}])$ and $\bK_{\bU\bU} := k([\bW_0^{\top}, ..., \bW_{T-1}^{\top}], [\bW_0^{\top}, ..., \bW_{T-1}^{\top}])$. Because for the inter-domain inducing points $k(w_t, w_{t'})=0$, the matrix $\bK_{\bU\bU}$ is block diagonal.
%
%The predictive mean can be written as,
%\begin{align}
%    \bK_{f\bU} \bK_{\bU\bU}^{-1} \vectorize{\bF \bM_v} = \sum_{t=0}^{T-1} \bK_{f \bu_t} \bK_{\bu_t \bu_t}^{-1} 
%\end{align}
%
% 
%\begin{align}
%    q^{inter}(f) = \normal(\bK_{f\bU} \bK_{\bU\bU}^{-1} \vectorize{\bF \bM_v}, \bK_{ff} - \bK_{f\bU} \bK_{\bU\bU}^{-1} () \bK_{\bU\bU}^{-1} \bK_{\bU f} )
%\end{align}
%
\end{proof}

\begin{lemma}[Variational Gaussian Approximations]\label{lem:vari-gau-apprx} Let $\normal(\bmu, \bS)$ be a Gaussian variational posterior for a SVGP, then the optimal $\bS^{\star}$ is in the form of,
\begin{align}
    \bS^{\star} =  \bK_{\bu \bu} \left( \bK_{\bu \bu} + \bK_{\bu \bbf} \bLambda \bK_{\bbf \bu} \right)^{-1} \bK_{\bu \bu}.
\end{align}
where $\bLambda = \diag([\lambda_n]_{n=1}^N)$ is diagonal, 
\begin{align}
    \lambda_n = -2\nabla_{\sigma^2_n}\expect_{q(f_n)} [\log p(y_n | f_n)],
\end{align}
where $\sigma^2_n$ is the predictive variance of $f_n$ under the variational posterior.
\end{lemma}
\begin{proof}[Proof]
Given the variational posterior, the predictive distribution of $f_n$ can be computed as,
\begin{align}
    \normal(\bk_{f \bu} \bK_{\bu\bu}^{-1}\bmu, k_{ff} +  \bk_{f \bu}\bK_{\bu\bu}^{-1} (\bS - \bK_{\bu\bu})\bK_{\bu\bu}^{-1} \bk_{\bu f}), \notag
\end{align}
where we denote the predictive variance as $\sigma_n^2$. The variational posterior is optimized by maximizing the ELBO, which can be computed as,
\begin{align}
    \loss = \sum_{n=1}^N \expect_{q(f_n)} [\log p(y_n | f_n)] - \KL{\normal(\bmu, \bS)}{\normal(\bzero, \bK_{\bu\bu})}, \notag 
\end{align}
We compute the derivatives of $\loss$ towards $\bS$,
\begin{align}
    \nabla_{\bS} \loss &= -\frac{1}{2}\sum_{n=1}^N \lambda_n \nabla_{\bS} \sigma_n^2 - \frac{1}{2}(\bK_{\bu\bu}^{-1} - \bS^{-1}) 
    = -\frac{1}{2}\bK_{\bu\bu}^{-1} (\sum_{n=1}^N \lambda_n \bk_{ \bu f_n} \bk_{f_n \bu})\bK_{\bu\bu}^{-1} - \frac{1}{2}(\bK_{\bu\bu}^{-1} - \bS^{-1}) \notag \\
    &=  -\frac{1}{2}\bK_{\bu\bu}^{-1} \bK_{\bu\bbf} \bLambda \bK_{\bbf \bu} \bK_{\bu\bu}^{-1} - \frac{1}{2} \bK_{\bu\bu}^{-1} + \frac{1}{2}\bS^{-1},
\end{align}
Let the derivative be zero, we obtain the optimal $\bS^{\star}$,
\begin{align}
    \bS^{\star} = \bK_{\bu\bu}(\bK_{\bu\bu} + \bK_{\bu\bbf} \bLambda \bK_{\bbf \bu})^{-1}\bK_{\bu\bu}. \notag 
\end{align}
\end{proof}

\begin{proof}[\textbf{Proof of Theorem~\ref{lem:decoupled}}] Based on Lemma~\ref{lem:vari-gau-apprx}, the optimal posterior covariance is 
\begin{align}
    \bS^{\star} = \bK_{\bu \bu} \left( \bK_{\bu \bu} +\bK_{\bu \bbf}\bLambda \bK_{\bbf \bu} \right)^{-1} \bK_{\bu \bu},
\end{align}
Given that $\bK_{\bu \bu}$ is block diagonal, by the continuous mapping theorem, it remains to prove that $\bK_{\bu \bbf}\bLambda  \bK_{\bbf \bu}$ approaches block diagonal.

Firstly we assume that Hermitian kernels are not resolved, thus $\bK_{\bbf\bu} = \bK_{\bu\bbf}^H$.
Because $\lambda_n$ only depends on $(\bx_n, y_n)$, for the $(\bz_t, \bz_{t'})$ off-diagonal element in $\bK_{\bu \bbf}\bLambda \bK_{\bbf \bu} $, 
\begin{align}
    &\frac{1}{N}\sum_{n=1}^N \lambda_n k_t(\bz_t, \bx_n) k_{t'}^H(\bz_{t'}, \bx_n) \to \expect_{p(\bx)p(y|\bx)}[ \lambda(\bx, y) k_t(\bz_t, \bx)k^H_{t'}(\bz_{t'}, \bx) ] 
    = \expect_{p(\bx)}[ \expect_{p(y|\bx)}[\lambda(\bx, y)] k_t(\bz_t, \bx)k^H_{t'}(\bz_{t'}, \bx) ],  \notag 
\end{align}
We let $\hat{\lambda}(\bx) := \expect_{p(y|\bx)}[\lambda(\bx, y)]$, then the formula can be further computed as,
\begin{align}
    & \expect_{p(\bx)}[ \hat{\lambda}(\bx) k_t(\bz_t, \bx)k^H_{t'}(\bz_{t'}, \bx) ] \notag \\
    &=\expect_{p(\bx)}[ \hat{\lambda}(\bx) \sum_{s=0}^{T-1} \sum_{s'=0}^{T-1} \bF_{t,s} \bF^H_{t',s'} k(\bx, G^s (\bz_t)) k^H(\bx, G^{s'} (\bz_{t'}))  ] \notag \\
    &= \expect_{p(\bx)}[ \hat{\lambda}(\bx) \sum_{s=0}^{T-1} \sum_{s'=0}^{T-1} \bF_{t,s} \bF^H_{t',s+s'} k(\bx, G^s(\bz_t)) k^H(\bx, G^{s+s'}(\bz_{t'}))  ] \notag \\
    &= \expect_{p(\bx)}[\hat{\lambda}(G^s(\bx))  \sum_{s=0}^{T-1} \sum_{s'=0}^{T-1}\bF_{t,s} \bF^H_{t',s+s'}k(G^{s} (\bx), G^{s}(\bz_t)) k^H(G^{s} (\bx), G^{s+s'} (\bz_{t'}))  ] \notag \\
    &=\expect_{p(\bx)}[\hat{\lambda}(G^s(\bx)) k( \bx, \bz_t) \sum_{s'=0}^{T-1} k^H(\bx, G^{s'} (\bz_{t'}))  \sum_{s=0}^{T-1} \bF_{t,s} \bF^H_{t',s+s'} ] = 0, \notag
\end{align}
In the second equality we used the periodicity of $G$; In the third equality we used the assumption that $G^s(\bx)$ has the same distribution as $\bx$; In the last equality we used the property that $ \sum_{s=0}^{T-1} \bF_{t,s} \bF^H_{t',s+s'}  = 0$ for all $t \neq t'$.

Furthermore, if the Hermitian kernels are resolved in HVGP, let $T$ be the period, then $\bK_{\bu \bbf}\bLambda \bK_{\bbf \bu}$ is a matrix of $(1+\lfloor T/2 \rfloor) \times (1+\lfloor T/2 \rfloor)$. For any off-diagonal element at $(t, t')$, $[\bK_{\bu \bbf}\bLambda \bK_{\bbf \bu}]_{t,t'}$ equals to,
\begin{align}
    \frac{1}{N}\sum_{n=1}^N \lambda_n \left(k_t(\bz_t, \bx_n) + k_{T-t}(\bz_t, \bx_n)\right)\left(k^H_{t'}(\bz_{t'}, \bx_n) + k^H_{T-t'}(\bz_{t'}, \bx_n)\right) \notag 
\end{align}
Given previous results, because $t, T-t$ are both different with $t', T-t'$, the formula becomes $0$ as well, as $N \to \infty$.
\end{proof}